\documentclass{article}
\usepackage{amsmath,amsthm,amssymb,mathrsfs}
\usepackage[utf8]{inputenc}
\usepackage{natbib} 
\usepackage[colorlinks,linkcolor=magenta,citecolor=blue, pagebackref=true]{hyperref}

\usepackage{algorithm}
\usepackage{algorithmic}
\usepackage{graphicx}
\usepackage{booktabs}
\usepackage{setspace}
\usepackage{authblk}

\newcommand{\R}{\mathbb{R}}
\newcommand{\E}{\mathbb{E}}
\newcommand{\I}{\mathbb{I}}
\newcommand{\mle}{\textrm{mle}}
\DeclareMathOperator*{\argmax}{argmax}

\newtheorem{theorem}{Theorem}
\newtheorem{lemma}{Lemma}
\newtheorem{definition}{Definition}

\newcommand{\cvec}[2]
{
\left[\begin{array}{c}
#1 \\ #2 
\end{array}\right] 
}

\newcommand{\symmat}[3]
{
\left[\begin{array}{cc}
#1 & #2 \\ #2 & #3
\end{array}\right] 
}

\makeatletter
\newcommand{\subalign}[1]{%
  \vcenter{%
    \Let@ \restore@math@cr \default@tag
    \baselineskip\fontdimen10 \scriptfont\tw@
    \advance\baselineskip\fontdimen12 \scriptfont\tw@
    \lineskip\thr@@\fontdimen8 \scriptfont\thr@@
    \lineskiplimit\lineskip
    \ialign{\hfil$\m@th\scriptstyle##$&$\m@th\scriptstyle{}##$\hfil\crcr
      #1\crcr
    }%
  }%
}
\makeatother

\DeclareMathOperator\supp{supp}

\title{Off-policy Confidence Sequences}
\author[1]{Nikos Karampatziakis}
\author[1]{Paul Mineiro}
\author[2]{Aaditya Ramdas}
\affil[1]{Microsoft}
\affil[2]{Carnegie Mellon University}
\affil[ ]{\texttt {\{nikosk,pmineiro\}@microsoft.com\ aramdas@cmu.edu}}
\date{\today}

\begin{document}

\maketitle
\begin{abstract}
	We develop confidence bounds that hold uniformly over time for off-policy
	evaluation in the contextual bandit setting. These confidence sequences
	are based on recent ideas from martingale analysis and are
	non-asymptotic, non-parametric, and valid at arbitrary stopping times.
	We provide algorithms for computing these confidence sequences that
	strike a good balance between computational and statistical efficiency.
	We empirically demonstrate the tightness of our approach in terms of
	failure probability and width and apply it to the ``gated deployment'' problem of safely upgrading a production contextual bandit system.
\end{abstract}

\section{Introduction} 
Reasoning about the reward that a new policy $\pi$ would have achieved if it
had been deployed, a task known as Off-Policy Evaluation (OPE), is one of the
key challenges in modern Contextual Bandits (CBs)~\cite{epochgreedy} and Reinforcement Learning (RL).  A typical OPE use case is the validation of new modeling ideas
by data scientists. If OPE suggests that $\pi$ is better, this can then be
validated online by deploying the new policy to the real world.

The classic way to to answer whether $\pi$ has better reward than the current
policy $h$ is via a confidence interval (CI).  Unfortunately, CIs take a very
static view of the world.  Suppose that $\pi$ is better than $h$ and our OPE
shows a higher but not significantly better estimated reward. What should we
do? We could collect more data, but since a CI holds for a particular (fixed) sample size and is not designed to handle interactive/adaptive data collection, simply recalculating the CI at a larger sample size invalidates its coverage
guarantee. 

While there are ways to fix this, such as a crude union bound, the proper statistical tool for such cases is called a Confidence Sequence (CS). A CS is a
sequence of CIs such that the probability that they ever exclude the true value
is bounded by a prespecified quantity.  In other words, they retain validity under optional (early)
stopping and optional continuation (collecting more data).

In this work we develop CSs for OPE using recent insights from martingale
analysis (for simpler problems).  Besides the aforementioned high probability uniformly over time
guarantee, these CSs make no parametric assumptions and are easy to compute.
We use them to create a ``gated deployment'' primitive: instead of deploying
$\pi$ directly we keep it in a staging area where we compute its off-policy CS
as $h$ is collecting data.  Then $\pi$ can replace $h$ as soon as (if ever) we
can reject the hypothesis that $h$ is better than $\pi$.

We now introduce some notation to give context to our contributions.  We have
iid CB data of the form $(x,a,r)$ collected by a historical
policy $h$ in the following way. First a context $x$ was sampled from an unknown
distribution $D$.  Then $h$ assigns a probability to each action.  An action
$a$ is sampled with probability $h(a;x)$ and performed. A reward $r$ associated
with performing $a$ in situation $x$ is sampled from an unknown distribution
$R(x,a)$.  Afterwards, we wish to estimate the reward of another policy $\pi \ll h$. We have
\begin{equation}
\label{eq:change-of-measure}
V(\pi) = \E_{\subalign{x&\sim D\\a&\sim\pi(x)\\r&\sim R(x,a)}}[r]
=
\E_{\subalign{x&\sim D\\a&\sim h(x)\\r&\sim R(x,a)}}\left[\frac{\pi(a;x)}{h(a;x)}r\right]
\end{equation}
where the last quantity can be estimated from data.
%
%\footnote{We assume absolute continuity $\pi \ll h$ 
%as typical in OPE} 
Letting $w=\frac{\pi(a;x)}{h(a;x)}$  we see that 
$\E_{x\sim D,a\sim h}[w]=1$, where we write $w$ instead of $w(x,a)$ to reduce notation clutter. More generally
for any function $q(x,a)$ --- which is typically a predictor of the reward of $a$ at $x$ --- we have
\begin{equation}
\label{eq:control-variate}
\E_{x\sim D,a\sim h}[w q(x,a)] = \sum_{a'} \pi(a';x) q(x,a'),    
\end{equation}
which reduces to $\E[w]=1$ when $q(x,a)=1$ always.
Eq.~\eqref{eq:change-of-measure} and \eqref{eq:control-variate} are the
building blocks of OPE estimators.  The IPS estimator \cite{HT52} estimates
\eqref{eq:change-of-measure} via Monte Carlo: $\hat{V}^{\textrm{IPS}}(\pi) =
1/n \sum_{i=1}^n w_i r_i $.  A plethora of other OPE estimators are discussed
in Section~\ref{sec:related}. In general there is a tension between the
desirability of an unbiased estimator like $\hat V^{\textrm{IPS}}$ and the
difficulty of working with it in finite samples due to its excessive
variance.

Recently, \citet{kallus2019intrinsically} proposed an OPE estimator based on
Empirical Likelihood~\cite{owen2001empirical} with several desirable
properties. Empirical Likelihood (EL) has also been used to derive CIs for OPE
in CBs \cite{karampatziakis2019empirical} and RL \cite{dai2020coindice}. Our
CSs can be thought of as a natural extension to the online setting of the CIs
for OPE in the batch setting; its advantages include
\begin{itemize}\setlength\itemsep{-0.1em}
\item Our CSs hold non-asymptotically, unlike most existing CIs mentioned above which are either asymptotically valid (or nonasymptotic but overly conservative).
\item Our CSs are not unnecessarily conservative due to naive union bounds or peeling techniques.
\item We do not make any assumptions, either parametric 
or about the support of $w$ and $r$, beyond boundedness.
\item Our validity guarantees are time-uniform, meaning that they remain valid under optional continuation (collecting more data) and/or at stopping times, both of which are not true for all aforementioned CIs.
\end{itemize}

\section{Background: OPE Confidence Intervals}
We start by reviewing OPE CIs from the perspective of
\citet{karampatziakis2019empirical}. Their CI is constructed by considering
plausible distributions from a nonparametric family $\mathcal{Q}$ of
distributions  $Q$ for random vectors $(w,r) \in [0,w_{\max}]\times [0,1]$
under the constraint $\E_Q[w]=1$. Let $Q_{wr}$ be the probability that $Q \in
\mathcal{Q}$ assigns to the event where the importance weight is $w$ and the
reward is $r$. Then there exists $Q^* \in \mathcal{Q}$ such that
\[
Q^*_{wr}=\E_{x\sim D,a\sim h,\rho\sim R(x,a)}
\left[
\I\left[\frac{\pi(a;x)}{h(a;x)}=w\right]\cdot
\I\left[\rho=r\right]
\right]
\]
and 
$
V(\pi)=\E_{Q^*}[wr]
$.
To estimate of $V(\pi)$ we can find $Q^{\mle} \in \mathcal{Q}$ that maximizes
the data likelihood. To find a CI we minimize/maximize $\E_Q[wr]$ over
plausible $Q \in \mathcal{Q}$ so the data likelihood is not far off from that
of $Q^{\mle}$.

Using convex duality the MLE is
$
Q^{\mle}_{wr} = \frac{1}{n(1+\lambda_1^{\mle}(w-1))}
$
where $\lambda_1^{\mle}$ is a dual variable solving
\[
\lambda_1^{\mle} = \argmax_{\lambda_1} \sum_{i=1}^n \log(1+\lambda_1(w_i-1))
\]
subject to $1+\lambda_1(w_{\max}-1)\geq 0$, $1-\lambda_1\geq 0$.  The profile
likelihood
$
L(v)=\sup_{Q: \E_Q[w]=1, \E_Q[wr]=v} \prod_{i=1}^n Q_{w_i,r_i}
$
is used for CIs. From EL theory, an asymptotic $1-\alpha$-CI is
\[
\left\{v: -2\ln\left(\frac{\prod_{i=1}^n Q^{\mle}_{w_i,r_i}}{L(v)}\right)
\leq \chi_1^{2,1-\alpha}\right\}
\]
where $\chi_1^{2,1-\alpha}$ is the $1-\alpha$ quantile of a $\chi^2$
distribution with one degree of freedom.  Using convex duality the CI is
\[
\left\{v: 
B(v)-\sum_{i=1}^n\log(1+\lambda_1^{\mle}(w_i-1))
\leq \chi_1^{2,1-\alpha}\right\}
\]
where the dual profile log likelihood $B(v)$ is
\begin{equation}
B(v) = \sup_{\lambda_1,\lambda_2} \sum_{i=1}^n \log(1+\lambda_1(w_i-1)+\lambda_2(w_i r_i -v))    \label{eq:dual-lik}
\end{equation}
subject to $(\lambda_1,\lambda_2) \in \mathcal{D}_v^0$ where 
\begin{align}
\mathcal{D}_v^{m} =
\{(\lambda_1,\lambda_2): & 1+\lambda_1(w-1)+\lambda_2(wr-v)\geq m \label{eq:batch-domain}\notag\\
                         & \forall (w,r) \in \{0,w_{\max}\}\times \{0,1\}
\}.
\end{align}
The CI endpoints can be found via bisection on $v$.

\section{Off-policy Confidence Sequences}
We now move from the batch setting and asymptotics to online procedures and
finite sample, time-uniform results.  We adapt and extend ideas from
\citet{waudby-smith_variance-adaptive_2020} which constructs 
CSs for the means of random variables in $[0,1]$.  
Our key insight is to combine their construction with 
an interpretation of \eqref{eq:dual-lik} 
as the log wealth accrued by a
skeptic who is betting against the hypotheses 
\[
\E_{Q^*}[w]=1 \text{ and } \E_{Q^*}[wr]=v.
\]
In particular, the skeptic starts with a wealth of $1$ and wants to maximize
her wealth. Her bet on the outcome $w-1$ is captured by $\lambda_1$, while
$\lambda_2$ represents the bet on the outcome of $wr-v$ so that the wealth
after the $i$-th sample is multiplied by $1+\lambda_1(w_i-1)+\lambda_2 (w_i r_i
-v)$. If the outcomes had been in $[-1,1]$ then $|\lambda_1|$ and $|\lambda_2|$
would have an interpretation as the fraction of the skeptic's wealth that is
being risked on each step. The bets can be positive or negative, and their
signs represent the directions of the bet. For example, $\lambda_2<0$ means the
skeptic will make money if $w_ir_i-v<0$.  Enforcing the constraints
\eqref{eq:batch-domain} from the batch setting here means that the resulting
wealth cannot be negative.

The first benefit of this framing is that we have mapped the abstract concepts
of dual likekihood, dual variables, and dual constraints to more familiar
concepts of wealth, bets, and avoiding bankruptcy.  We now formalize our
constructions and show how they lead to always valid, finite sample, CSs. We
introduce a family of processes
\[
K_t(v) = \prod_{i=1}^t (1+\lambda_{1,i} (w_i-1) +\lambda_{2,i}(w_i r_i - v))
\]
where $\lambda_{1,i}$ and $\lambda_{2,i}$ are predictable, i.e. based on past data (formally, measurable with
respect to the sigma field $\sigma(\{(w_j,r_j)\}_{j=1}^{i-1})$).  We also formalize CIs and CSs below.
\begin{definition}
Given data $S_n=\{(x_i,a_i,r_i)\}_{i=1}^n$, where 
$x_i \sim D$, $a_i\sim h(\cdot;x_i)$,  $r_i \sim R(x_i,a_i)$, 
a $(1-\alpha)$-confidence interval 
for $V(\pi)$ is a set $C_n = C(h,\pi,S_n)$ such that
\[
\sup_{D,R} \Pr(V(\pi) \notin C_n) \leq \alpha.
\]
In contrast, a $(1-\alpha)$-confidence sequence for $V(\pi)$ 
is a sequence of confidence intervals $(C_t)_{t \in \mathbb{N}}$ such that
\[
\sup_{D,R} \Pr(\exists t \in \mathbb{N}: V(\pi) \notin C_t) \leq \alpha.
\]
\end{definition}

We now have the setup to state our first theoretical result.
\begin{theorem}
\label{thm:martingale}
\label{thm:ville}
\label{thm:cs}
$K_t(V(\pi))$ is a nonnegative martingale. Moreover,
the sequences $C_t = \{v:K_t(v)\leq \frac{1}{\alpha}\}$ 
and $\mathfrak{C}_t = \bigcap_{i=1}^t C_i$
are $(1-\alpha)$-confidence sequences for $V(\pi)$.
\end{theorem}

All proofs are in the appendix.  
The process $K_t(v)$
tracks the wealth of a skeptic betting against $V(\pi)=v$. The process
$K_t(V(\pi))$ is a nonnegative martingale so it has a small probability of attaining large values (formally, Ville's inequality states that the probability of ever exceeding $1/\alpha$ is at most $\alpha$). 
Of course, we don't know $V(\pi)$, but if we retain all values of $v$ where the wealth is below $1/\alpha$, and reject the values of $v$ for which it has crossed $1/\alpha$ at some point, this set will always contain $V(\pi)$ with high probability; this is the basis of our construction.
The strength of our approach comes from this result, 
as it guarantees always-valid bounds for $V(\pi)$ using 
only martingale arguments crucially avoiding
parametric or other assumptions.

What about $v \neq V(\pi)$? Can we be sure that $C_t$ does not contain values $v$ very far from $V(\pi)$? That's where the betting strategy, quantified by the predictable sequences $(\lambda_{1,i})$ and $(\lambda_{2,i})$, enters. The hope is the skeptic can eventually 
force $K_t(v)$ to be large via a series of good bets. 
Importantly, Theorem~\ref{thm:cs} holds regardless of how the bets are set,
but good bets will lead to ``small'' $C_t$.
How to smartly bet is the subject of what
follows.

\section{Main Betting Strategy: MOPE}
We develop our main betting strategy, MOPE (Martingale OPE) 
in steps starting with a slow but effective algorithm 
and making changes to trade off a small amount of statistical
efficiency for large gains in computational efficiency.

\subsection{Follow The Leader}
%\citet{waudby-smith_variance-adaptive_2020} develop an array of
%increasingly more effective betting procedures.
We begin with a Follow-The-Leader
(FTL) strategy that is known 
to work very well for iid problems~\cite{de2014follow}.
We define $\ell_i^v(\lambda)=\ln(1+\lambda_1 (w_i-1)+\lambda_2(w_i r_i -
v))$ and set $\lambda = [\lambda_1, \lambda_2]$ to maximize 
wealth in
hindsight 
\begin{equation}
\lambda_{t}^{\textrm{ftl}}(v) = \argmax_{\lambda} \sum_{i=1}^{t-1}
\ell_{i}^{v}(\lambda)
\label{eq:ftl}
\end{equation}
for every step of betting in $K_t(v)$.  
The problem~\eqref{eq:ftl} is convex and can be solved in
polynomial time leading to an overall polynomial time algorithm. 
However, this approach has three undesirable properties. 
First, the algorithm needs to store
the whole history of $(w,r)$ samples. 
Second the overall algorithm is
tractable but slow.  
Finally, we need to solve~\eqref{eq:ftl} for all
values of $v$ that we have not yet rejected.

\subsection{Maximizing a lower bound on wealth}

We can avoid having to store all history by optimizing an easy-to-maintain
lower bound of \eqref{eq:ftl}.  
\begin{lemma} 
\label{lem:quadbound}
For all $x\geq -\frac{1}{2}$ and $\psi=2-4\ln(2)$, we have
\[
\ln(1+x)\geq x + \psi x^2.
\]
\end{lemma}

Observe that if we restrict our bets to lie in the convex set
$\mathcal{D}_v^{1/2}$ (cf.~eq.~\eqref{eq:batch-domain}) then for all $\lambda
\in \mathcal{D}_v^{1/2}$
\begin{align*}
\sum_{i=1}^{t-1} \ell_i^v(\lambda) 
\geq
\lambda^\top \sum_{i=1}^{t-1} b_i(v) 
+\psi \lambda^\top \left(\sum_{i=1}^{t-1} A_i(v)\right) \lambda
\end{align*}
where 
$b_i(v)=
\left[\begin{array}{c} 
w_i-1 \\ w_i r_i -v 
\end{array}\right] 
$
and 
$
A_i(v) = b_i(v)b_i(v)^\top.
$
The first step towards a more efficient algorithm is to set our bets at time
$t$ as
\begin{equation}
\lambda_t(v) = \argmax_{\lambda \in \mathcal{D}_{v}^{1/2}}
\psi  \lambda^\top \left(\sum_{i=1}^{t-1} A_i(v)\right) \lambda 
+ \lambda^\top \sum_{i=1}^{t-1} b_i(v)
\label{eq:quadoptv}
\end{equation}
The restriction $\lambda \in \mathcal{D}_{v}^{1/2}$ is very mild: it does not
allow the skeptic to lose more than half of
her wealth from any single outcome.  The first advantage of this formulation is
that $\sum_i A_i(v)$ and $\sum_i b_i(v)$ are low degree polynomials of $v$ and
can share the coefficients
    \begin{align*}
        \sum_{i=1}^{t-1} A_i(v) &= 
        A_t^{(0)} + v A_t^{(1)} + v^2 A_t^{(2)}\\   
        \sum_{i=1}^{t-1} b_i(v) &= b_t^{(0)} + v b_t^{(1)}.  
    \end{align*}
Secondly, the coefficients can be updated incrementally
\allowdisplaybreaks
    \begin{align}
        A_t^{(0)} &=\sum_{i=1}^{t-1}\symmat{(w_i-1)^2}{(w_i-1)w_i r_i}{w_i^2r_i^2}, \label{eq:upsuffa0}\\
        A_t^{(1)} &= \sum_{i=1}^{t-1} \symmat{0}{-(w_i-1)}{-2w_ir_i},\\
        A_t^{(2)} &=\sum_{i=1}^{t-1}  \symmat{0}{0}{1},\\
        b_t^{(0)} &=\sum_{i=1}^{t-1}  \cvec{w_i-1}{w_ir_i},\\
        b_t^{(1)} &=\sum_{i=1}^{t-1}  \cvec{0}{-1}. \label{eq:upsuffb1}
    \end{align}
Finally, we can solve~\eqref{eq:quadoptv} exactly in $O(1)$ time.
Section~\ref{sec:avoid-grid} will elaborate on this using a slight variation of
eq.~\eqref{eq:quadoptv}.

\subsection{Common Bets and Hedging}
\label{sec:hedged}
The most competitive betting sequences for the process $K(v)$ will take
advantage of the knowledge of $v$. However, placing different bets for
different values of $v$ creates two problems: First, the resulting confidence
set need not be an interval and second makes it hard to implement
Theorem~\ref{thm:cs} in a computationally efficient way.  Indeed, even in the
simpler setup of \citet{waudby-smith_variance-adaptive_2020} the authors
maintain a grid of test values for the quantity of interest (here $v$)  and at
least keep track of the wealth separately.  This is because tracking the wealth
for each value in the grid is not straightforward when the bets are different. 

To make wealth tracking easy and obtain algorithms that do not require the
discretization of the domain of $v$, a natural proposal would be to use a
common bet for all $v$ in each timestep. Unfortunately, this is not adequate
because we do need $\lambda_2 > 0$ for $v<\E_{Q^*}[wr]$ and $\lambda_2 < 0$ for
$v>\E_{Q^*}[wr]$.  A simple fix is to use a hedged strategy as
in~\citet{waudby-smith_variance-adaptive_2020}.  First, we split the initial
wealth equally.  The first half is used to bet against low $v$'s via the process
\[
K_t^{+}(v) = \prod_{i=1}^t \left(1+\lambda_{1,i}^{+}(w_i-1)+\lambda_{2,i}^{+}(w_i r_i -v)\right)
\]
and the second half to bet against high $v$'s via a separate process
$K_t^{-}(v)$ which for symmetry we parametrize as
\[
K_t^{-}(v) = \prod_{i=1}^t \left(1+\lambda_{1,i}^{-}(w_i-1)+\lambda_{2,i}^{-}(w_i r'_i -v')\right).
\]
where $r'_i=1-r_i$ and $v'=1-v$.  This can be seen as the wealth process for
betting against $1-v$ in an world where $r$ has been remapped to $1-r$.  Thus
betting against high values of $v$ reduces to betting against low values of $v$
in a modified process. The total wealth of the hedged process is
\begin{equation}
K_t^{\pm}(v) = \frac{1}{2} (K_t^{+}(v) + K_t^{-}(v)),
\label{eq:hedged}
\end{equation}
and it can be used for CSs in the same way as $K_t(v)$:
\begin{theorem}
\label{thm:hedged}
The sequence $C_t^{\pm} = \{v:K_t^{\pm}(v)\leq \frac{1}{\alpha}\}$ and its 
running intersection $\bigcap_{i=1}^t C_i^{\pm}$
are $1-\alpha$ CSs for $V(\pi)$.
\end{theorem}

It remains to design a common bet for $K_t^{+}(v)$.  Betting against any fixed
$v_0$ will not work well when $V(\pi)=v_0$ since the optimal bet for $V(\pi)$
is 0 but such a bet cannot help us reject those $v$ that are far from $V(\pi)$.
Therefore we propose to adaptively choose the bets against the smallest $v$
that has not been rejected. As we construct the CS, we have access to the values of $v$ that constitute the
endpoints of the CS at the last time step. These values are on
the cusp of plausibility given the available data and confidence level which
means the bets are neither too conservative nor too detached from what can be
estimated.

\subsection{Avoiding grid search}
\label{sec:avoid-grid}
Once we have determined $v$ for the current step we could choose $\lambda$ via
\eqref{eq:quadoptv}. For reasons that will become apparent shortly, we can also
consider
\begin{equation}
\lambda_t = \argmax_{\lambda \in \mathcal{C}}
\psi  \lambda^\top \left(\sum_{i=1}^{t-1} A_i(v)\right) \lambda 
+ \lambda^\top \sum_{i=1}^{t-1} b_i(v),
\label{eq:quadopt}
\end{equation}
where $\mathcal{C}=\{\lambda: \lambda_2 \geq 0\}\cap \bigcap_{v\in [0,1]} \mathcal{D}_v^{1/2}$
or more succintly
$$
\mathcal{C} = \left\{\lambda: \lambda_2\geq 0, 
\lambda_1 + \lambda_2 \leq \frac{1}{2},
\lambda_1 \left(1-w_{\max}\right) + \lambda_2 \leq  \frac{1}{2}
\right\}.
$$
The constraint $\lambda_2\geq 0$ is expected for good bets in $K_t^+(v)$ (and
by reduction in $K_t^-(v)$) since we are eliminating $v$'s with $\E[wr-v]>0$.
Since there are only three constraints and two variables we can exactly solve
\eqref{eq:quadopt} very efficiently. Our implementation first tries to
return the unconstrained maximizer, if feasible.  If not, we evaluate the
objective on up to 6 candidates: up to one candidate per face of $\mathcal{C}$
(obtained via maximizing the objective subject to one equality constraint) and
its 3 vertices. Algorithm~\ref{alg:argmax} summarizes this.

\begin{algorithm}[tb]
   \caption{Solve $\lambda^* = \argmax_{\lambda \in \mathcal{C}} \psi \lambda^\top A \lambda + \lambda^\top b$}
   \label{alg:argmax}
\begin{algorithmic}
    \STATE {\bfseries Input:} $A, b$
    \STATE $\lambda = -(2\psi A)^{-1} b$
    \IF {$\lambda \in \mathcal{C}$}
        \STATE \textbf{Return} $\lambda$
    \ENDIF
    \STATE $\Lambda = \left\{\left[\frac{1}{2(1-w_{\max})},0\right],\left[\frac12,0\right],\left[0,\frac12\right]\right\}$     \COMMENT{vertices of $\mathcal{C}$}
    \FOR{$c,d \in \{([0,1],0), ([1,1],\frac12), ([1-w_{\max},1],\frac{1}{2})\}$}
    \STATE $\mu = -\frac{c^\top (2\psi A)^{-1}b+d}{c^\top (2\psi A)^{-1}c}$
    \COMMENT {Lagrange multiplier}
    \STATE $\lambda = -(2\psi A)^{-1} (b + \mu c)$ 
    \IF {$\lambda \in \mathcal{C}$}
    \STATE $\Lambda = \Lambda \cup \{\lambda\}$ 
    \COMMENT{Add feasible solutions on faces}
    \ENDIF 
    \ENDFOR
    \STATE \textbf{Return}  $\argmax_{\lambda \in \Lambda} \psi \lambda^\top A \lambda + \lambda^\top b$
\end{algorithmic}
\end{algorithm}

Given $\lambda_1,\ldots,\lambda_{t-1}$, from \eqref{eq:quadopt} 
we get from Lemma~\ref{lem:quadbound} 
\[
\sum_{i=1}^{t-1} \ell_i^v(\lambda_i)
\geq
\psi  
 \sum_{i=1}^{t-1} \lambda_i^\top A_i(v)\lambda_i  +  \sum_{i=1}^{t-1} \lambda_i^\top b_i(v)
\]
for all $v \in [0,1]$. Thus, if the lower bound exceeds $\ln(1/\alpha)$ 
for a particular $v$, the log wealth will also exceed it. Furthermore,
the lower bound is quadratic in $v$ so we can easily find those values
$v \in [0,1]$ such that
\begin{equation}
   \psi  
 \sum_{i=1}^{t-1} \lambda_i^\top A_i(v)\lambda_i  +  \sum_{i=1}^{t-1} \lambda_i^\top b_i(v) = \ln\left(\frac{2}{\alpha}\right).
 \label{eq:quadv}
\end{equation} 
The extra $2$ is due to the hedged process.  Appendix~\ref{app:nogrid2d}
explains this and the details of how to incrementally maintain statistics for solving
\eqref{eq:quadv} via eqs.~\eqref{eq:upsuffc}-\eqref{eq:upsuffu}.  The advantage
of \eqref{eq:quadopt} over \eqref{eq:quadoptv} is that the latter cannot ensure
that old bets will produce values in $\mathcal{D}_v^{1/2}$ for future values of
$v$ while the former always does because $\mathcal{C} \subseteq
\mathcal{D}_{v}^{1/2},~ \forall v \in [0,1]$. 

The whole process of updating the statistics, tightening the lower bound $v$
via \eqref{eq:quadv} and computing the new bets via \eqref{eq:quadopt} is
summarized in Algorithm~\ref{alg:main}.

\begin{algorithm}[tb]
   \caption{MOPE: Martingale Off-Policy Evaluation}
   \label{alg:main}
\begin{algorithmic}
    \STATE {\bfseries Input:} process $Z=(w_i,r_i)_{i=1}^\infty, w_{\max}, \alpha$
    \STATE Let $Z' = (w_i,1-r_i)$ for $(w_i,r_i)$ in $Z$
    \FOR{$v_i,v_i'$ in zip(LCS($Z$), LCS($Z'$))}
        \STATE Output($v_i,1-v_i'$)
   \ENDFOR
\FUNCTION{LCS($Z$)}{
    \STATE $\lambda_1 = [0,0]^\top, v = 0$
    \FOR{$i=1,\ldots$}
        \STATE Observe $(w_i,r_i)$ from $Z$
        \STATE Update statistics via \eqref{eq:upsuffa0}-\eqref{eq:upsuffb1} and \eqref{eq:upsuffc}-\eqref{eq:upsuffu}.
        \IF{\eqref{eq:quadv} has real roots}
            \STATE $v=\max(v, \textrm{largest root of }\eqref{eq:quadv})$
        \ENDIF
        \STATE \textbf{yield} $v$
        \COMMENT{execution suspends/resumes here}
        \STATE $A = A_i^{(0)} +v A_i^{(1)} + v^2 A_i^{(2)}$
        \STATE $b = b_i^{(0)} +v b_i^{(1)}$
        \STATE  $\lambda_{i+1} = \argmax_{\lambda \in C} \psi  \lambda^\top A \lambda+  b^\top \lambda$.
    \ENDFOR
    }
\ENDFUNCTION
\end{algorithmic}
\end{algorithm}

\paragraph{Confidence Intervals.} If one only desires a single CI using a fixed batch of data, then a CI can be formed by returning the last set from the CS on any permutation of the data.  To reduce variance, we can average the wealth of several independent permutations without violating validity.

\paragraph{Alternative Betting Algorithms}
An obvious question is why develop this strategy 
and not just feed the 
convex functions $-\ell_i^v(\lambda)$ to an online 
learning algorithm? 
The Online Newton Step
(ONS)~\cite{hazan2007logarithmic}
is particularly well-suited
as $-\ell_i^v(\lambda)$ is exp-concave.
While ONS does not require storing all history and needs
small per-step computation, we could not find an 
efficient way to efficiently reason about 
$K_t(v)$ for every $v \in [0,1]$.
While the ONS bounds the log wealth in terms
of the gradients observed at each bet, these gradients depend on $v$ in a way
that makes it hard to efficiently reuse for different values of $v$. 
Our approach on the other hand maintains a lower bound on the wealth 
as a second degree polynomial in $v$, enabling us to reason about 
all values of $v$ in constant time. 
%Finding efficient online algorithms for parametrized families of
%problems is an interesting direction for future work.

\section{Extensions}

\subsection{Adding a Reward Predictor}
\label{sec:reward-predictor}
So far we have only used the special case $\E[w]=1$ of
eq.~\eqref{eq:control-variate}. However, it is common to have access to a
reward predictor $q(x,a)$ mapping context and action to an estimated reward.
Here we will just assume that $q(x,a)$ is any measurable function of $(x,a)$
with codomain $[0,1]$. We use eq.~\eqref{eq:control-variate} to define the 
zero mean quantity
\begin{equation}
c_i=w_i q(x_i,a_i)-\sum_{a'}\pi(a';x_i)q(x_i,a').    
\label{eq:ci-defn}
\end{equation}
Thus $\E[wr-c]=\E[wr]$ but when $q(x_i, a_i)$ is a good 
reward predictor, the variance of $wr - c$ will be much 
smaller than that of $wr$. We propose the wealth process
\[
K_t^q(v)=\prod_{i=1}^{t}\left(1+\lambda_{1,i} (w-1)+\lambda_{2,i}(w_i r_i-c_i-v)\right)
\]
for predictable sequences 
$(\lambda_{1,i}, \lambda_{2,i}) \in \mathcal{E}_{v}^0$, where
\begin{align*}
\mathcal{E}_v^{m} =
\{(\lambda_1,\lambda_2):  1+\lambda_1(w-1)+\lambda_2(wr-c-v)\geq m \notag\\
\forall (x,a,r,q) \in \supp(D) \times \mathcal{A} \times \{0,1\} \times [0,1]^{|\mathcal{A}|}
\}.
\end{align*}
Note that $w=w(x, a)$ and $c=c(x,a,q)$ so all quantities are well defined.
This set looks daunting but without loss of generality it suffices to
only consider two 
actions: $a$, which is sampled by $h$, and an alternative one $a'$,
$h(a) \in \{1/w_{\max},1\}$, and $\pi(a),q(x,a),q(x,a') \in \{0,1\}$.
Considering all these combinations and removing redundant constraints 
leads to the equivalent description for $\mathcal{E}_v^{m}$ as
\begin{align}
\tiny
\left\{\lambda:  
%\left[\begin{array}{cc}
%-1 & -v \\ 
%-1 & v' \\
%W & -W-v\\
%W & W+v'
%\end{array}\right]
\left[\begin{array}{cccc}
-1 & -1 & W & W \\ 
-v & v' & -W-v & W+v'
\end{array}\right]^\top
\lambda \geq m - 1 \label{eq:rp-explicit-domain}
\right\},
\end{align}
where $W = w_{\max}-1$ and $v'=1-v$.

For an efficient procedure we introduce the set 
$\mathcal{C}^q=\bigcap_{v \in [0,1]}\mathcal{E}_v^{1/2}$
to enable the use of our lower bound and common bets for all $v$.
This set can be shown to be the same as \eqref{eq:rp-explicit-domain}
but with $v=1$ and $v'=1$. For predictable sequences of 
bets $\lambda_{t}^{+q},\lambda_{t}^{-q}\in \mathcal{C}^q$ 
define the processes
\begin{align*}
K_t^{+q}(v) = \prod_{i=1}^{t}\left(1+\lambda_{1,i}^{+q} (w_i-1)+\lambda_{2,i}^{+q}(w_i r_i-c_i-v)\right),    \\
K_t^{-q}(v) = \prod_{i=1}^{t}\left(1+\lambda_{1,i}^{-q} (w_i-1)+\lambda_{2,i}^{-q} (w_i r_i'-c_i'-v')\right), 
\end{align*}
where $r_i'=1-r_i$, $v'=1-v$. For the definition of
$c_i'$ we reason as follows: If $q(x,a)$ is a good reward predictor
for $r_i$ then $q'(x,a)=1-q(x,a)$ is a good reward predictor for $r_i'$. 
Plugging $q'(x,a)$ in place of $q(x,a)$ in \eqref{eq:ci-defn} leads 
to $c_i' = w_i - 1 - c_i$. Finally, the hedged process is just
$
K_t^{\pm q}(v)=\frac{1}{2} (K_t^{+q}(v)+K_t^{-q}(v))
$
and we have
\begin{theorem}
\label{thm:reward-predictor}
The sequences $C_t^{q} = \{v:K_t^{q}(v)\leq \frac{1}{\alpha}\}$ and 
$C_t^{\pm q} = \{v:K_t^{\pm q}(v)\leq \frac{1}{\alpha}\}$ as well as
their running intersections $\bigcap_{i=1}^t C_i^{q}$ and
$\bigcap_{i=1}^t C_i^{\pm q}$ are $1-\alpha$ CSs for $V(\pi)$.
\end{theorem}
Appendix~\ref{app:reward-predictors} contains the details on how to bet.
We close this section with two remarks. First, a ``bad'' $q(x,a)$ can
make $wr-c$ have larger variance than $wr$. To protect against this case
we can run a \emph{doubly hedged} process: 
$K_{t}^{\pm^2}(v)=\frac{1}{2}
(K_{t}^{\pm q}(v)+K_{t}^{\pm}(v))$
which will accrue wealth almost as well as the best
of its two components.
Second our framework allows for $q(x,a)$ to be updated in every step as long as the updates are predictable. 

\subsection{Scalar Betting}
\label{sec:scalar}
Since $\E[w]=1$, it would seem that the $\lambda_1$ bet cannot have any long
term benefits. While this will be shown to be false 
in our experiments we nevertheless develop a betting strategy that only
bets on $w_i r_i -v$. The advantages of this strategy are computational 
and conceptual simplicity.
Similarly to Section~\ref{sec:hedged} we use a hedged
process $K_t^{\gtrless}(v)=\frac{1}{2}\left(K_t^{>}(v)+K_t^{<}(v)\right)$ where
\begin{align*}
K_t^{>}(v)&=\prod_{i=1} \left(1+\lambda_{2,i}^{>} (w_i r_i -v)\right),\\
K_t^{<}(v)&=\prod_{i=1} \left(1+\lambda_{2,i}^{<} \left(w_i (1-r_i) -(1-v)\right)\right).
\end{align*}
Appendix~\ref{app:betaopt} provides an alternative justification via a worst
case argument.  The upshot is that $\lambda_1=\max(0,-\lambda_2)$ is a
reasonable choice and it leads to the above processes.

We explain betting for $K_t^{>}(v)$, since betting for $K_t^{<}(v)$ reduces to
that.  We use a result by \citet{fan2015exponential}:
\[
\ln(1+\lambda \xi) \geq \lambda \xi+\left(\ln\left(1-\lambda\right)+\lambda\right)\cdot \xi^{2}
\]
for all $\xi\geq -1$ and $\lambda \in [0,1)$, which we reproduce in
Appendix~\ref{app:fan}.  We apply it in our case with $\xi_i=w_ir_i-v\geq -1$
and consider the log wealth lower bound for a fixed $\lambda_2$
\[
\ln(K_t^{>}(v)) \geq \lambda_2 \sum_{i=1}^{t-1} \xi_i + \left(\ln\left(1-\lambda_2\right)+\lambda_2\right) \sum_{i=1}^{t-1} \xi_i^2.
\]
When $\sum_{i=1}^{t-1} \xi_i^2>0$ the lower bound is concave and can 
be maximized in $\lambda_2$ by setting its derivative to 0. This gives
\[
\lambda_{2,t}^{>}(v) = \frac{\sum_{i=1}^{t-1} (w_i r_i -v)}{\sum_{i=1}^{t-1} (w_i r_i -v)+\sum_{i=1}^{t-1} (w_i r_i -v)^2}.
\]
When $\sum_{i=1}^{t-1} \xi_i^2=0$ we can set $\lambda_{2,t}^{>}(v)=0$.
Finally, employing the same ideas as Section~\ref{sec:avoid-grid} we can
adaptively choose the $v$ to bet against and avoid maintaining a grid of values
for $v$. Details are in Appendix~\ref{app:nogrid1d}

\subsection{Gated Deployment}
A common OPE use case is to estimate the difference $V(\pi) - V(h)$. If we
can reject all negative values (i.e.\ the lower CS crosses 0) then $\pi$ should be deployed. Conversely,
rejecting all positive values (i.e.\ the upper CS crosses 0) means $\pi$ should be discarded. Since $h$ is the policy collecting the data we have $V(h)=\E[r]$. Thus we can form a CS 
around $V(\pi) - V(h)$ by considering the process:
\[
K_t^{gd}(v) = \prod_{i=1}^t \left(1+\lambda_{1,i} (w_i-1) + \lambda_{2,i}(w_ir_i -r_i -v)\right)  
\]
for predictable $\lambda_{1,i},\lambda_{2,i}$ subject to $(\lambda_{1,i},\lambda_{2,i}) \in \mathcal{G}_{v}^0$
where
\begin{align}
\label{eq:gd-domain}
\mathcal{G}_{v}^m = \{(\lambda_{1},\lambda_{2}): & 1+\lambda_1(w-1)+\lambda_2(wr-r)\geq m \notag\\
                         & \forall (w,r) \in \{0,w_{\max}\}\times \{0,1\}
\}.
\end{align}
As before, we can form a hedged process and restrict bets to a set that
enables the use of our lower bound. We defer these details to 
appendix~\ref{app:gd}. We can then show
\begin{theorem}
\label{thm:gated}
The sequences $C_t^{gd} = \{v:K_t^{gd}(v)\leq \frac{1}{\alpha}\}$ and $\bigcap_{i=1}^t C_i^{gd}$
are $1-\alpha$ CSs for $V(\pi)-V(h)$.
\end{theorem}
This CS has two advantages 
over a classical A/B test. First, we don't have to 
choose a stopping time in advance. The CS can run for
as little or as long as necessary. Second, if $\pi$ were
worse than $h$ the A/B test would have an adverse effect
on the quality of the overall system, while here we can
reason about this degradation without deploying $\pi$.

\section{Related Work}
\label{sec:related}

Apart from IPS, other popular OPE estimators include Doubly Robust \cite{RnR,
dudik2011doubly} which incorporates \eqref{eq:control-variate} as an additive
control variate and SNIPS \cite{swaminathan2015self} which incorporates
$\E[w]=1$ as a multiplicative control variate.  The quest to balance the
bias-variance tradeoff in OPE has led to many different 
proposals \cite{SWITCH,vlassis2019design}.
EL-based estimators are proposed in~\citet{kallus2019intrinsically} and
\citet{karampatziakis2019empirical}.

CIs for OPE include both finite-sample~\cite{thomas2015high} and
asymptotic~\cite{li2015counterfactual,karampatziakis2019empirical} ones. Some
works that propose both types are~\citet{bottou2013counterfactual} and
\citet{dai2020coindice}. The latter obtains CIs without knowledge of $w$, a much
more challenging scenario that requires additional assumptions.

We are not aware of any CSs for OPE.  For on-policy setups, the most
competitive CSs all rely on exploiting (super)martingales, and in some sense
all admissible CSs have to~\cite{ramdas2020admissible}.  Examples include Robbins' mixture
martingale \cite{robbins_statistical_1970} and the techniques of
\citet{howard_uniform_2019}. The recent work
of~\citet{waudby-smith_variance-adaptive_2020} which 
leverages a betting view substantially
increases the scope of these techniques, while
simplifying and tightening the constructions. Similar betting 
ideas have been recently used in the development of 
parameter-free online algorithms \cite{OrabonaP16}.

\section{Experiments}

\begin{figure}[h!]
    \centering
    \includegraphics[width=0.75\linewidth]{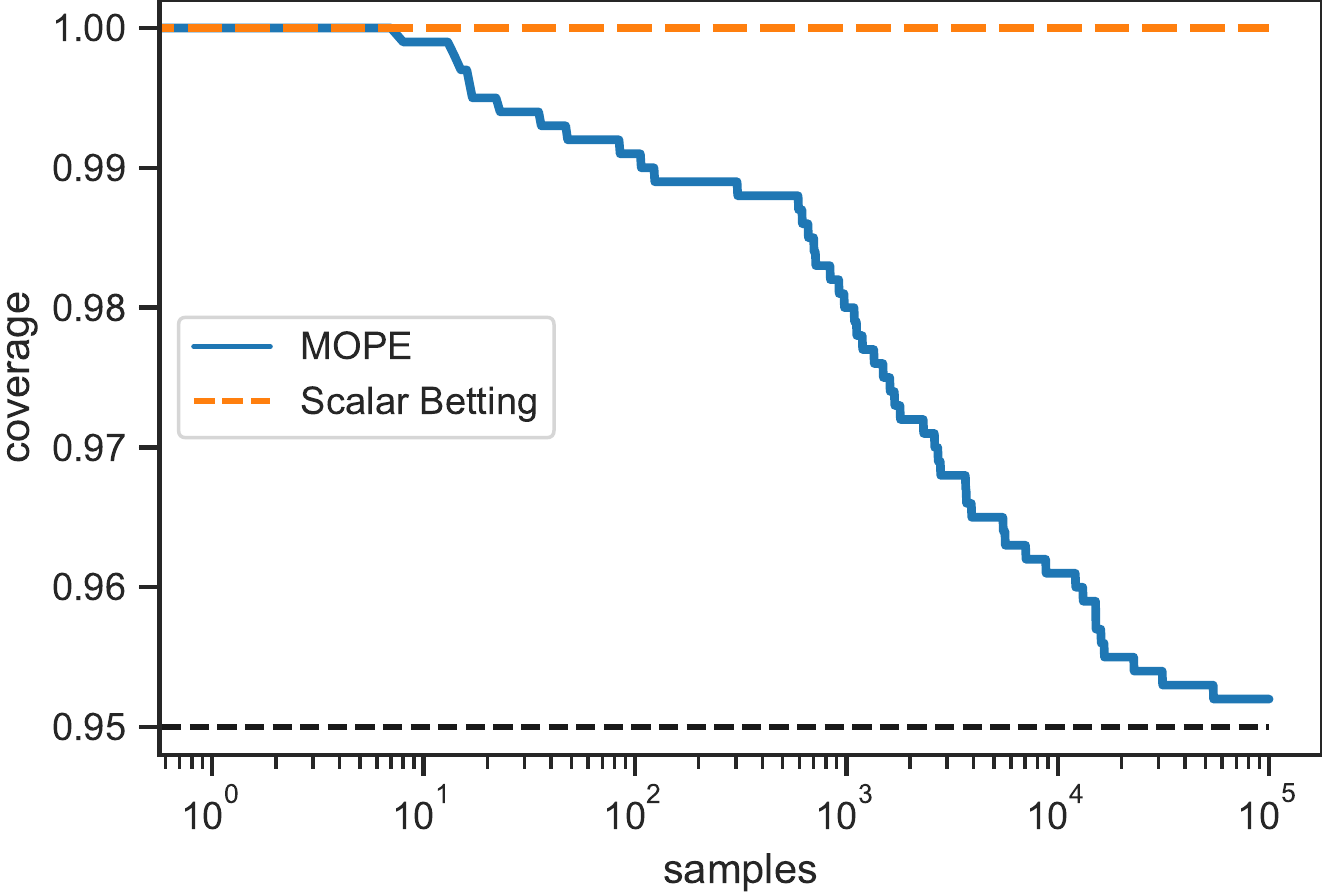}
    \caption{Empirical coverage for two proposed CSs. The CS that bets on 
    both $w-1$ and $wr-v$ converges to nominal coverage while the CS
    that does not bet on $w-1$ overcovers.}
    \label{fig:coverage}
\end{figure}

\begin{figure*}
    \centering
    \includegraphics[width=0.8\textwidth]{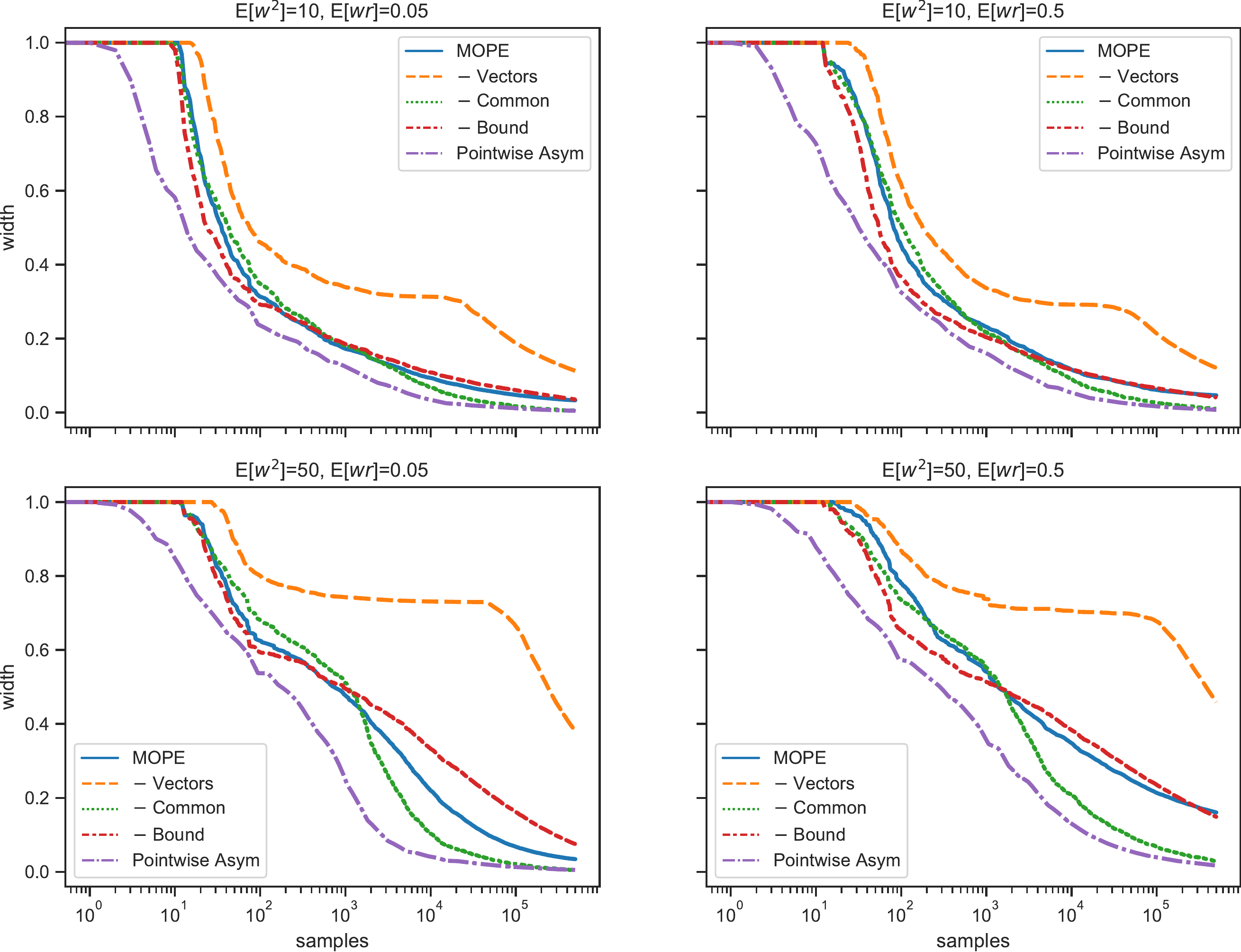}
    \caption{The width of 95\% CS produced by MOPE and its three ablations.
    The pointwise asymptotic curve is \emph{not} a CS.}
    \label{fig:width}
\end{figure*}

\begin{figure}
    \centering
    \includegraphics[width=0.75\linewidth]{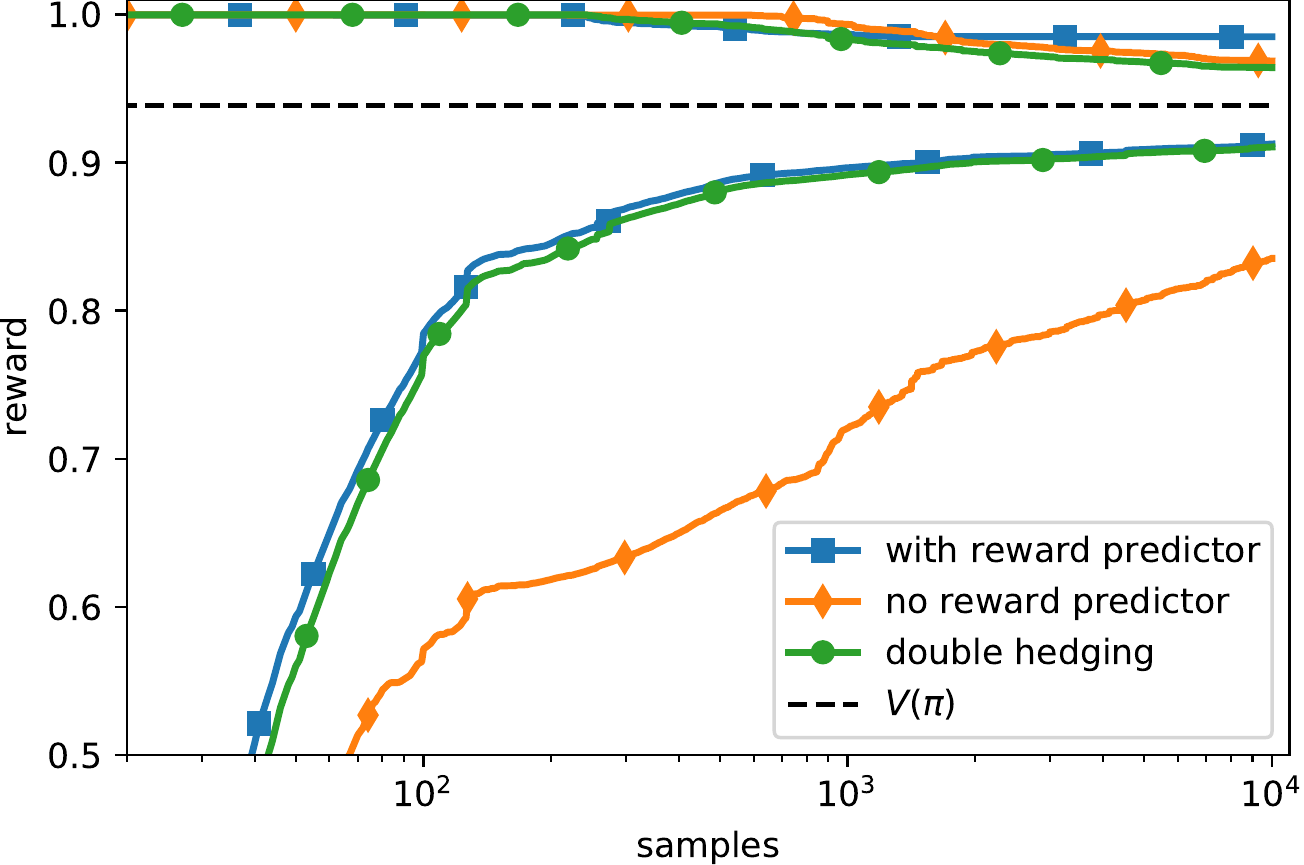}
    \caption{Three 99.9\% CSs with/without a reward predictor and a doubly hedged one that achieves the best of both}
    \label{fig:predictor}
\end{figure}

\begin{figure}
    \centering
    \includegraphics[width=0.75\linewidth]{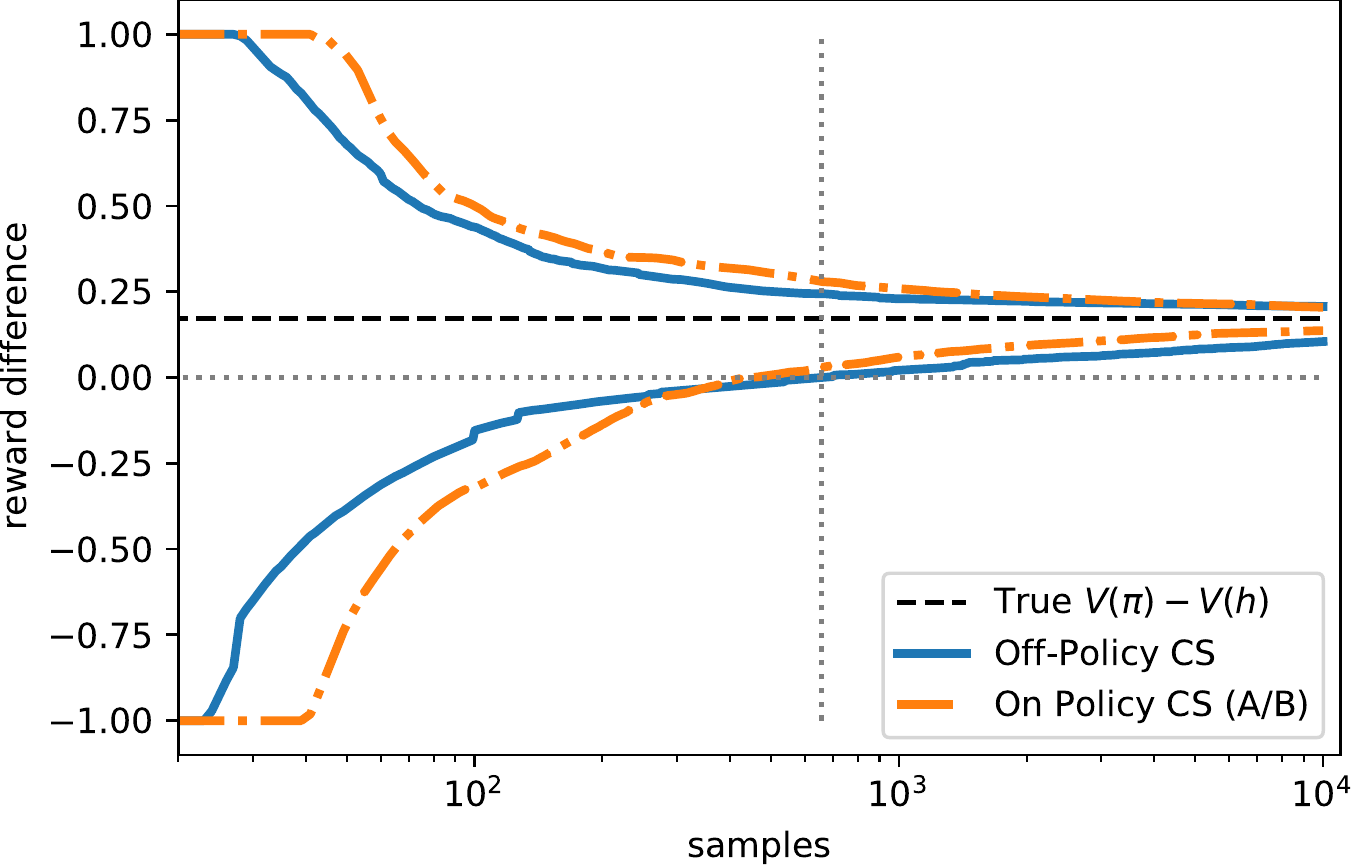}
    \caption{CS for gated deployment and A/B test. $\pi$ can be 
    deployed as soon as the lower CS crosses 0 (dotted line at $t=657$).}
    \label{fig:gd}
\end{figure}

\subsection{Coverage} \label{sec:coverage}

While any predictable betting sequence
guarantees correct coverage, some will overcover more than
others. Here we investigate the coverage properties of
MOPE and the strategy of Section~\ref{sec:scalar}.
We generate 1000 sequences of 100000 $(w,r)$ pairs each from a different
distribution. All distributions are maximum entropy distributions subject to
$(w,r) \in \{0, 0.5, 2, 100\} \times \{0,1\}$, $\E[w]=1$, $\E[w^2]=10$
and $V(\pi)$ sampled uniformly in $[0,1]$. In Figure~\ref{fig:coverage} we show
the empirical mean coverage of the two CSs for $\alpha=0.05$. 
MOPE approaches nominal coverage from above,
a property rarely seen with standard confidence bounds.

\subsection{Computational vs. Statistical Efficiency}

We run an ablation study for the three ingredients of MOPE, where
{\sf $-$Vector} is the scalar betting technique of section~\ref{sec:scalar}; {\sf $-$Common} solves \eqref{eq:quadoptv} over a grid of 200 $v$ values at each timestep;
and {\sf $-$Bound} optimizes the log wealth exactly rather than the bound of Lemma~\ref{lem:quadbound}, i.e., Algorithm~\ref{alg:main} with equation~\eqref{eq:ftl} in lieu of equation~\eqref{eq:quadopt}.

We use four synthetic environments which are distributions over $(w,r)$ generated in the same way as section~\ref{sec:coverage} but
with $(V(\pi),\E[w^2]) \in \{0.05, 0.5\} \times \{10, 50\}$.
Table~\ref{tab:timings}
shows the running times for each method in the environment with 
the largest variance. We see that directly 
maximizing wealth and individual betting per $v$ are very slow.
MOPE and {\sf $-$Vector} are computationally efficient.
In Figure~\ref{fig:width} we show the 
average CS width over 10 repetitions for 500000 time steps
for MOPE and its ablations as well as the 
asymptotic CI from~\citet{karampatziakis2019empirical}
which is only valid \emph{pointwise} and provides a lower bound for
all CSs in the figure. MOPE is better than {\sf $-$Vector} 
and as good or better than {\sf $-$Bound}.
While MOPE is not as tight as
the (much more computationally demanding) 
{\sf $-$Common}, the gap is small in all 
but the most challenging environment.

\begin{table}[t]
\caption{Timings for MOPE and its ablations on 500000 samples}
\label{tab:timings}
\centering
\begin{small}
\begin{sc}
\begin{tabular}{lcccr}
\toprule
Method & MOPE & -Vector & -Common & -Bound \\
\midrule
Time (sec)& 32     & 14.5  & 10440 & 15882 \\
\bottomrule
\end{tabular}
\end{sc}
\end{small}
\vspace{-0.05in}
\end{table}

\subsection{Effect of a Reward Predictor}
\label{sec:rp-experiment}
We now investigate the use of  
reward predictors in our CSs using the processes
$K_t^{\pm q}(v)$ and $K_t^{\pm^2}(v)$ 
of Section~\ref{sec:reward-predictor}.
We use the first 1 million samples from the mnist8m
dataset which has 10 classes and train the following
functions: $h$ using linear multinomial 
logistic regression (MLR), $\pi$ again using MLR but
now on $1000$ random Fourier 
features (RFF)~\cite{rahimi2007random} that approximate 
a Gaussian kernel machine, and finally $q$ which
uses the same RFF represetation as $\pi$ but instead
its $i$-th output is independently trained 
to predict whether the input is the $i$-th class
using 10 binary logistic regressions.
We used the rest of the data with the following 
protocol: for each input/label pair $(x_i,y_i)$, we sample 
action $a_i$ with probability $0.9h(a_i;x_i)+0.01$
(so that we can safely set $w_{\max}=100$),
we set $r_i=1$ if $a_i=y_i$, otherwise $r_i=0$, 
and record $w_i$ and $c_i$. We estimated 
$V(\pi)\approx 0.9385$ using the next million samples. In Figure~\ref{fig:predictor}
we show the CS for $V(\pi)$ averaged over 5 runs each 
with 10000 different samples using the processes 
$K_t^{\pm q}(v)$, $K_t^{\pm}(v)$ and $K_t^{\pm^2}(v)$.
We see that including 
a reward predictor dramatically improves the lower bound
and somewhat hurts the upper bound. 
The doubly hedged process on the other hand attains the 
best of both worlds.

\subsection{CSs for Gated Deployment}
Here we investigate the use of CSs for gated deployment. 
We use the same $h$ and $\pi$ and the same data as 
in Section~\ref{sec:rp-experiment} but now we are 
using the process $K_t^{gd}(v)$ 
(or rather a computationally efficient version 
of this process based on a hedged process with 
common bets and optimizing a quadratic lower bound
c.f.\ Appendix~\ref{app:gd}).
Figure~\ref{fig:gd} shows the average CS over 5 runs
each with 10000 different samples. We see that the 
CS contains the true difference (about 0.17) and quickly
decides that $\pi$ is better than $h$ at $t=657$ samples.
We also include an on-policy CS (from \cite{waudby-smith_variance-adaptive_2020}) which
\emph{can only be computed if $\pi$ is deployed} 
e.g.\ in an A/B test. While this is riskier 
when $\pi$ is inferior to $h$, the on-policy rewards 
typically have lower variance. Thus the on policy CS can 
conclude that $\pi$ is better than $h$ at the 
same $\alpha=0.01$ using 440 samples 
(220 for each policy). If the roles of 
$\pi$ and $h$ were swapped, i.e,\ $\pi$ was the 
behavior policy and $h$ was a proposed alternative,
the on policy CS would still need to collect 220
samples from $h$. In contrast, a system using the off-policy CS 
would never have to experience any regret when 
the behavior policy is superior.

\section{Conclusions}
We presented a generic way to construct confidence sequences 
for OPE in the Contextual Bandit setting. The construction leaves 
a lot of freedom in designing betting strategies and we mostly 
explored options with an eye towards computational
efficiency. Theoretically we achieve finite sample coverage 
and validity at any time with minimal assumptions. Empirically
the resulting sequences are tight and not too far away from asymptotic 
and pointwise valid existing work. Theoretical results on the 
width of our CSs remain elusive and are both an interesting 
area for future work and a key to unlock much stronger 
analyses of various algorithms in Bandits and RL.

\newpage

\bibliography{opecs}

\begin{thebibliography}{24}
\providecommand{\natexlab}[1]{#1}
\providecommand{\url}[1]{\texttt{#1}}
\expandafter\ifx\csname urlstyle\endcsname\relax
  \providecommand{\doi}[1]{doi: #1}\else
  \providecommand{\doi}{doi: \begingroup \urlstyle{rm}\Url}\fi

\bibitem[Langford and Zhang(2007)]{epochgreedy}
John Langford and Tong Zhang.
\newblock The epoch-greedy algorithm for contextual multi-armed bandits.
\newblock In \emph{Proceedings of the 20th International Conference on Neural
  Information Processing Systems}, pages 817--824. Citeseer, 2007.

\bibitem[Horvitz and Thompson(1952)]{HT52}
Daniel~G Horvitz and Donovan~J Thompson.
\newblock A generalization of sampling without replacement from a finite
  universe.
\newblock \emph{Journal of the American statistical Association}, 47\penalty0
  (260):\penalty0 663--685, 1952.

\bibitem[Kallus and Uehara(2019)]{kallus2019intrinsically}
Nathan Kallus and Masatoshi Uehara.
\newblock Intrinsically efficient, stable, and bounded off-policy evaluation
  for reinforcement learning.
\newblock \emph{arXiv preprint arXiv:1906.03735}, 2019.

\bibitem[Owen(2001)]{owen2001empirical}
{Art B} Owen.
\newblock \emph{Empirical likelihood}.
\newblock Chapman and Hall/CRC, 2001.

\bibitem[Karampatziakis et~al.(2020)Karampatziakis, Langford, and
  Mineiro]{karampatziakis2019empirical}
Nikos Karampatziakis, John Langford, and Paul Mineiro.
\newblock Empirical likelihood for contextual bandits.
\newblock \emph{Advances in neural information processing systems}, 33, 2020.

\bibitem[Dai et~al.(2020)Dai, Nachum, Chow, Li, Szepesvari, and
  Schuurmans]{dai2020coindice}
Bo~Dai, Ofir Nachum, Yinlam Chow, Lihong Li, Csaba Szepesvari, and Dale
  Schuurmans.
\newblock Coindice: Off-policy confidence interval estimation.
\newblock \emph{Advances in neural information processing systems}, 33, 2020.

\bibitem[Waudby-Smith and Ramdas(2020)]{waudby-smith_variance-adaptive_2020}
Ian Waudby-Smith and Aaditya Ramdas.
\newblock Variance-adaptive confidence sequences by betting.
\newblock \emph{arXiv:2010.09686 [math, stat]}, October 2020.
\newblock URL \url{http://arxiv.org/abs/2010.09686v1}.
\newblock arXiv: 2010.09686.

\bibitem[De~Rooij et~al.(2014)De~Rooij, Van~Erven, Gr{\"u}nwald, and
  Koolen]{de2014follow}
Steven De~Rooij, Tim Van~Erven, Peter~D Gr{\"u}nwald, and Wouter~M Koolen.
\newblock Follow the leader if you can, hedge if you must.
\newblock \emph{The Journal of Machine Learning Research}, 15\penalty0
  (1):\penalty0 1281--1316, 2014.

\bibitem[Hazan et~al.(2007)Hazan, Agarwal, and Kale]{hazan2007logarithmic}
Elad Hazan, Amit Agarwal, and Satyen Kale.
\newblock Logarithmic regret algorithms for online convex optimization.
\newblock \emph{Machine Learning}, 69\penalty0 (2-3):\penalty0 169--192, 2007.

\bibitem[Fan et~al.(2015)Fan, Grama, and Liu]{fan2015exponential}
Xiequan Fan, Ion Grama, and Quansheng Liu.
\newblock Exponential inequalities for martingales with applications.
\newblock \emph{Electronic Journal of Probability}, 20, 2015.

\bibitem[Robins and Rotnitzky(1995)]{RnR}
James~M. Robins and Andrea Rotnitzky.
\newblock Semiparametric efficiency in multivariate regression models with
  missing data.
\newblock \emph{Journal of the American Statistical Association}, 90\penalty0
  (429):\penalty0 122--129, 1995.

\bibitem[Dud{\'\i}k et~al.(2011)Dud{\'\i}k, Langford, and Li]{dudik2011doubly}
Miroslav Dud{\'\i}k, John Langford, and Lihong Li.
\newblock Doubly robust policy evaluation and learning.
\newblock \emph{arXiv preprint arXiv:1103.4601}, 2011.

\bibitem[Swaminathan and Joachims(2015)]{swaminathan2015self}
Adith Swaminathan and Thorsten Joachims.
\newblock The self-normalized estimator for counterfactual learning.
\newblock In \emph{advances in neural information processing systems}, pages
  3231--3239, 2015.

\bibitem[Wang et~al.(2017)Wang, Agarwal, and Dud{\'{\i}}k]{SWITCH}
Yu{-}Xiang Wang, Alekh Agarwal, and Miroslav Dud{\'{\i}}k.
\newblock Optimal and adaptive off-policy evaluation in contextual bandits.
\newblock In \emph{Proceedings of the 34th International Conference on Machine
  Learning, {ICML} 2017, Sydney, NSW, Australia, 6-11 August 2017}, pages
  3589--3597, 2017.
\newblock URL \url{http://proceedings.mlr.press/v70/wang17a.html}.

\bibitem[Vlassis et~al.(2019)Vlassis, Bibaut, Dimakopoulou, and
  Jebara]{vlassis2019design}
Nikos Vlassis, Aurelien Bibaut, Maria Dimakopoulou, and Tony Jebara.
\newblock On the design of estimators for bandit off-policy evaluation.
\newblock In \emph{International Conference on Machine Learning}, pages
  6468--6476, 2019.

\bibitem[Thomas et~al.(2015)Thomas, Theocharous, and
  Ghavamzadeh]{thomas2015high}
Philip Thomas, Georgios Theocharous, and Mohammad Ghavamzadeh.
\newblock High-confidence off-policy evaluation.
\newblock In \emph{Proceedings of the AAAI Conference on Artificial
  Intelligence}, volume~29, 2015.

\bibitem[Li et~al.(2015)Li, Chen, Kleban, and Gupta]{li2015counterfactual}
Lihong Li, Shunbao Chen, Jim Kleban, and Ankur Gupta.
\newblock Counterfactual estimation and optimization of click metrics in search
  engines: A case study.
\newblock In \emph{Proceedings of the 24th International Conference on World
  Wide Web}, pages 929--934. ACM, 2015.

\bibitem[Bottou et~al.(2013)Bottou, Peters, Qui{\~n}onero-Candela, Charles,
  Chickering, Portugaly, Ray, Simard, and Snelson]{bottou2013counterfactual}
L{\'e}on Bottou, Jonas Peters, Joaquin Qui{\~n}onero-Candela, Denis~X Charles,
  D~Max Chickering, Elon Portugaly, Dipankar Ray, Patrice Simard, and
  Ed~Snelson.
\newblock Counterfactual reasoning and learning systems: The example of
  computational advertising.
\newblock \emph{The Journal of Machine Learning Research}, 14\penalty0
  (1):\penalty0 3207--3260, 2013.

\bibitem[Ramdas et~al.(2020)Ramdas, Ruf, Larsson, and
  Koolen]{ramdas2020admissible}
Aaditya Ramdas, Johannes Ruf, Martin Larsson, and Wouter Koolen.
\newblock Admissible anytime-valid sequential inference must rely on
  nonnegative martingales.
\newblock \emph{arXiv preprint arXiv:2009.03167}, 2020.

\bibitem[Robbins(1970)]{robbins_statistical_1970}
Herbert Robbins.
\newblock Statistical {Methods} {Related} to the {Law} of the {Iterated}
  {Logarithm}.
\newblock \emph{The Annals of Mathematical Statistics}, 41\penalty0
  (5):\penalty0 1397--1409, October 1970.
\newblock ISSN 0003-4851, 2168-8990.

\bibitem[Howard et~al.(2020)Howard, Ramdas, McAuliffe, and
  Sekhon]{howard_uniform_2019}
Steven~R. Howard, Aaditya Ramdas, Jon McAuliffe, and Jasjeet Sekhon.
\newblock Time-uniform, nonparametric, nonasymptotic confidence sequences.
\newblock \emph{The Annals of Statistics}, forthcoming, 2020.

\bibitem[Orabona and Pál(2016)]{OrabonaP16}
Francesco Orabona and Dávid Pál.
\newblock Coin betting and parameter-free online learning.
\newblock In \emph{NIPS}, pages 577--585, 2016.

\bibitem[Rahimi and Recht(2007)]{rahimi2007random}
Ali Rahimi and Benjamin Recht.
\newblock Random features for large-scale kernel machines.
\newblock In \emph{NIPS}, volume~3, page~5. Citeseer, 2007.

\bibitem[Ville(1939)]{ville1939etude}
Jean Ville.
\newblock Etude critique de la notion de collectif.
\newblock \emph{Bull. Amer. Math. Soc}, 45\penalty0 (11):\penalty0 824, 1939.

\end{thebibliography}
\bibliographystyle{unsrtnat}

\onecolumn
\appendix

\section{Proofs}

\subsection{Main Lemma}
The following lemma will be helpful in the proofs of all our Theorems.
\begin{lemma}
\label{lem:main}
Suppose we have a family of stochastic processes 
$(M_t(m))_{t=0}^{\infty}$ indexed by $m \in [0,1]$ and further assume
the process $(M_t(\mu))_{t=0}^{\infty}$ for some $\mu \in [0,1]$ is a 
non-negative martingale with respect to a filtration $\mathcal{F}_t$ 
(i.e.\ $\E[M_t|\mathcal{F}_{t-1}]=M_{t-1}$ for $t \geq 1$)
with initial value $M_0=1$. Then for any given
$\alpha \in [0,1]$ the sequence of sets
$C_t = \left\{m: M_t(m) \leq \frac{1}{\alpha}\right\}$ is a 
$(1-\alpha)$ confidence sequence for $\mu$ and so is
its running intersection $\bigcap_{i=1}^t C_i$.
\end{lemma}
\begin{proof}
For the first part, by the definition of a CS it suffices to show that
$\Pr\left(\exists t \in \mathbb{N}: \mu \notin C_t\right) \leq \alpha$ or
\[
\Pr\left(\exists t \in \mathbb{N}: \mu \notin \left\{m: M_t(m) \leq \frac{1}{\alpha}\right\}\right) \leq \alpha.
\]
An error occurs only if $M_t(\mu)$ exceeds $1/\alpha$ at any point. This means that it suffices to show that
\[
\Pr\left(\exists t \in \mathbb{N}: M_t(\mu) \geq \frac{1}{\alpha}\right) \leq \alpha ,
\]
which is true by Ville's inequality~\cite{ville1939etude} since $M_t(\mu)$ is a 
non-negative martingale with initial value 1.

For the second part, we need to show that
\[
\Pr\left(\exists t \in \mathbb{N}: \mu \notin \bigcap_{s=1}^t\left\{m: M_s(m) \leq \frac{1}{\alpha}\right\}\right) \leq \alpha.
\]
This reduces to showing
\[
\Pr\left(\exists t \in \mathbb{N}: \exists s\in \{1,\ldots,t\}: M_s(\mu) \geq \frac{1}{\alpha}\right) \leq \alpha,
\]
which further simplifies to
\[
\Pr\left(\exists t \in \mathbb{N}: M_t(\mu) \geq \frac{1}{\alpha}\right) \leq \alpha,
\]
and this is again implied by Ville's inequality.
\end{proof}

\subsection{Proof of Theorem~\ref{thm:martingale}}
\begin{proof}
Consider the filtration $(\mathcal{F}_t)_{t=0}^{\infty}$ generated
by the sequence of sigma-fields 
$\mathcal{F}_0 \subset \mathcal{F}_1 \subset \ldots$ 
with $\mathcal{F}_0$ the trivial sigma-field and 
$\mathcal{F}_t = \sigma((w_0,r_0),(w_1,r_1),\ldots,(w_t,r_t))$.
It suffices to show that our betting ensures that $K_t(V(\pi))$
is a non-negative martingale with initial value 1
as we can then apply lemma~\ref{lem:main}. $K_0(v)=1$ is by 
the definition of the process (we start with a wealth of 1),
and $K_t(v)\geq 0$ for all $v \in [0,1]$ because our bets 
are in the set $\mathcal{D}_v^0$ (c.f.\ eq~\eqref{eq:batch-domain}). 
Thus it remains to show
$\E\left[K_t\left(V(\pi)\right)|\mathcal{F}_{t-1}\right]=K_{t-1}(V(\pi))$.
We have the following chain of equalities
\begin{align*}
    \E\left[K_t\left(V(\pi)\right)|\mathcal{F}_{t-1}\right]&=
    \E\left[K_{t-1}\left(1+\lambda_{1,t}(w_t-1)+\lambda_{2,t}(w_t r_t -V(\pi))\right)|\mathcal{F}_{t-1}\right]\\
    &=K_{t-1} \E\left[1+\lambda_{1,t}(w_t-1)+\lambda_{2,t}(w_t r_t -V(\pi))|\mathcal{F}_{t-1}\right]\\
    &=K_{t-1} \left(1+\E\left[\lambda_{1,t}(w_t-1)|\mathcal{F}_{t-1}\right]+
    \E\left[\lambda_{2,t}(w_t r_t -V(\pi))|\mathcal{F}_{t-1}\right]\right)\\
    &=K_{t-1} \left(1+\lambda_{1,t}\E\left[(w_t-1)|\mathcal{F}_{t-1}\right]+
    \lambda_{2,t}\E\left[(w_t r_t -V(\pi))|\mathcal{F}_{t-1}\right]\right)\\
    &=K_{t-1} \left(1+\lambda_{1,t}\cdot 0 + \lambda_{2,t} \cdot 0\right) = K_{t-1}
\end{align*}
where we have used that $K_{t-1}$, $\lambda_{1,t}$, $\lambda_{2,t}$
are measurable with respect to $\mathcal{F}_{t-1}$ and that $\E[w]=1$ and 
$\E[wr]=V(\pi)$.
\end{proof}

\subsection{Proof of Lemma~\ref{lem:quadbound}}
\begin{proof}
Consider the function $f(x) = \ln(1+x)- x - \psi x^2$ with domain $\left[-\frac{1}{2},\infty\right)$. Note that $f\left(-\frac{1}{2}\right)=0$
and $\lim_{x\to\infty} f(x) = \infty$.
Furthermore $f$ has two critical points: $0$ and $-\frac{2\psi + 1}{2\psi}$. 
But $f(0)=0$ and $f\left(-\frac{2\psi + 1}{2\psi}\right)>0$ so we conclude 
that $f(x)\geq 0$ for all $x\geq -\frac{1}{2}$.
\end{proof}

\subsection{Proof of Theorem~\ref{thm:hedged}}
\begin{proof}
We will first show that $K_t^{\pm}(V(\pi))$ is a non-negative martingale with initial
value 1.
Consider the same filtration as for Theorem~\ref{thm:martingale}.
Note that $K_0^{\pm}(v)=1$ is by 
the definition of the process (we start with a wealth of 1).
We analyze $K_t^{+}(v)$ and $K_t^{-}(v)$ separately. 
Note that $K_t^{+}(v)\geq 0$ for all $v \in [0,1]$ because our bets 
are in the set $\mathcal{C} \subset \mathcal{D}_v^0$ (c.f.\ eq~\eqref{eq:batch-domain}).
For $K_t^{-}(v)$ we note that the process is isomorphic to a process
similar to $K_t^{+}(v)$ but with the reward and $v$ redefined. Thus
our bets keep $K_t^{-}(v)\geq 0$. 
We now show the equality
\begin{align*}
 \E\left[K_t^{-}\left(V(\pi)\right)|\mathcal{F}_{t-1}\right]&=K_{t-1}^{-}(V(\pi)),
\end{align*}
as the equality $\E\left[K_t^{+}\left(V(\pi)\right)|\mathcal{F}_{t-1}\right]=K_{t-1}^{+}(V(\pi))$ is exactly what was shown in Theorem~\ref{thm:martingale}.  
We have
\begin{align*}
    \E\left[K_t^{-}\left(V(\pi)\right)|\mathcal{F}_{t-1}\right]&=
    \E\left[K_{t-1}^{-}\left(1+\lambda_{1,t}^{-}(w_t-1)+\lambda_{2,t}^{-}(w_t (1-r_t) (1-V(\pi)))\right)|\mathcal{F}_{t-1}\right]\\
    &=K_{t-1}^{-} \E\left[1+\lambda_{1,t}^{-}(w_t-1)+\lambda_{2,t}^{-}(w_t - 1 + V(\pi) - w_t r_t)|\mathcal{F}_{t-1}\right]\\
    &=K_{t-1}^{-} \left(1+(\lambda_{1,t}^{-}+\lambda_{2,t}^{-})\E\left[(w_t-1)|\mathcal{F}_{t-1}\right]+
    \lambda_{2,t}^{-}\E\left[(V(\pi)-w_t r_t)|\mathcal{F}_{t-1}\right]\right)\\
    &=K_{t-1}^{-} \left(1+(\lambda_{1,t}^{-}+\lambda_{2,t}^{-})\cdot 0 + \lambda_{2,t}^{-} \cdot 0\right) = K_{t-1}^{-}.
\end{align*}
Therefore $\frac{1}{2}\left(K_t^{+}\left(V(\pi)\right)+K_t^{-}\left(V(\pi)\right)\right)$ is also a non-negative martingale with initial value 1. Applying 
Lemma~\ref{lem:main} finishes the proof of the theorem.
\end{proof}

\subsection{Proof of Theorem~\ref{thm:reward-predictor}}
\begin{proof}
We note that the proof below works for a sequence of 
predictable functions $q_t(x,a)$ but to reduce notation
we use $q(x,a)$.
Consider the filtration $(\mathcal{F}_t)_{t=0}^{\infty}$ generated
by the sequence of sigma-fields 
$\mathcal{F}_0 \subset \mathcal{F}_1 \subset \ldots$ 
with $\mathcal{F}_0$ the trivial sigma-field and 
$\mathcal{F}_t = \sigma((x_1,a_1,r_1),\ldots,(x_t,a_t,r_t))$.
Note that $K_0^{q}(v)=1$ and $K_t^{q}(v)\geq 0$ for all $v \in [0,1]$ 
because our bets are in the set $\mathcal{C}^q$.
Thus it remains to show
$\E\left[K_t^q\left(V(\pi)\right)|\mathcal{F}_{t-1}\right]=K_{t-1}(V(\pi))$.
We have the following chain of equalities
\begin{align*}
    \E\left[K_t^q\left(V(\pi)\right)|\mathcal{F}_{t-1}\right]&=
    \E\left[K_{t-1}^q\left(1+\lambda_{1,t}(w_t-1)+\lambda_{2,t}(w_t r_t - c_t -V(\pi))\right)|\mathcal{F}_{t-1}\right]\\
    &=K_{t-1}^q (1+\lambda_{1,t}\E\left[w_t-1|\mathcal{F}_{t-1}\right]
    +\lambda_{2,t}\E\left[w_t r_t -V(\pi)|\mathcal{F}_{t-1}\right]
    -\lambda_{2,t}\E\left[c_t|\mathcal{F}_{t-1}\right])\\
    &=K_{t-1}^q \left(\left.1-\lambda_{2,t}\E\left[w_t q(x_t,a_t) - \sum_{a'} \pi(a';x_t)q(x_t,a')\right|\mathcal{F}_{t-1}\right]\right) = K_{t-1}^q,
\end{align*}
where we have used that $K_{t-1}$, $\lambda_{1,t}$, $\lambda_{2,t}$
are measurable with respect to $\mathcal{F}_{t-1}$ and that $\E[w]=1$ and 
$\E[wr]=V(\pi)$ as well as $\E_{x_t\sim D,a_t\sim h}\left[w_t q(x_t,a_t)\right] = \E_{x_t}\left[\sum_{a'}\pi(a';x_t)q(x_t,a')\right]$.
Thus the claim for $C_t^q$ and its running intersection can be shown
by applying lemma~\ref{lem:main}. The claim for $C_t^{\pm q}$ is 
completely analogous using the ideas here and in the proof of 
Theorem~\ref{thm:hedged}.
\end{proof}

\subsection{Proof of Theorem~\ref{thm:gated}}
\begin{proof}

Consider the same filtration as for Theorem~\ref{thm:martingale}.
Note that $K_0^{gd}(v)=1$ is by 
the definition of the process and that $K_t^{gd}(v)\geq 0$ for all 
$v \in [0,1]$ because our bets 
are in the set $\mathcal{G}_v^0$ (c.f.\ eq~\eqref{eq:gd-domain}).
Finally, we have
%N.B.: the `small` block is for arxiv only
\begin{small}
\begin{align*}
    \E\left[K_t^{gd}\left(V(\pi)-V(h)\right)|\mathcal{F}_{t-1}\right]
    &= \E\left[K_{t-1}^{gd}\bigg(1+\lambda_{1,t}(w_t-1)+\lambda_{2,t}
    \big(w_t r_t - r_t - \left(V(\pi)-V(h)\right)\big)\bigg)|\mathcal{F}_{t-1}\right]\\
    &= K_{t-1}^{gd}\left(1+\lambda_{1,t} \E\left[w_t-1|\mathcal{F}_{t-1}\right]
    +\lambda_{2,t}\E\left[w_t r_t - V(\pi)|\mathcal{F}_{t-1}\right]
    -\lambda_{2,t}\E\left[r_t - V(h)|\mathcal{F}_{t-1}\right]\right)\\
    &= K_{t-1}^{gd}\left(1+\lambda_{1,t} \cdot 0 + \lambda_{2,t} \cdot 0 - \lambda_{2,t} \cdot 0\right) = K_{t-1}^{gd}.
\end{align*}
\end{small}
Therefore $K_t^{gd}(V(\pi)-V(h))$ is a non-negative martingale with initial
value 1. Applying 
lemma~\ref{lem:main} finishes the proof of the theorem.
\end{proof}

\section{Avoiding grid search}
\label{app:nogrid2d}
We first lower bound each process separately, then lower bound
the hedged process. We denote the bets for $K^{+}$ (respectively
$K^{-}$) as $\lambda^{+}$, (resp. $\lambda^{-}$).
From lemma~\ref{lem:quadbound} we have
\[
\ln(K_t^{+}(v)) \geq \sum_{i=1}^{t-1} {\lambda_i^{+}}^\top b_i(v) + \psi \sum_i {\lambda_i^{+}}^\top A_i(v) {\lambda_i^{+}},
\]
and
\[
\ln(K_t^{-}(v)) \geq \sum_{i=1}^{t-1} {\lambda_i^{-}}^\top b_i'(v') + \psi \sum_i {\lambda_i^{-}}^\top A_i'(v') {\lambda_i^{-}},
\]
where $v'=1-v$, 
$b_i'(v)=
\cvec{w_i-1}{w_i (1-r_i) -v}
$
and $A_i'(v)=b_i'(v)b_i'(v)^\top$.
For the Hedged process, using that for any $a,b$
\[
\ln\left(\exp(a)+\exp(b)\right)\geq \max(a,b)
\]
to first establish
\[
\ln(K^{\pm}(v)) \geq \max(\ln(K^{+}(v))-\ln(2),\ln(K^{-}(v))-\ln(2))
\]
and further bound each term in the maximum by the respective 
quadratic lower bound. We conclude that
if a $v$ achieves 
\[
 \sum_{i=1}^{t-1} {\lambda_i^{+}}^\top b_i(v) + \psi \sum_i {\lambda_i^{+}}^\top A_i(v) \lambda_i^{+} = \ln\left(\frac{2}{\alpha}\right),
\]
or a $v'=1-v$ achieves 
\[
\sum_{i=1}^{t-1} {\lambda_i^{-}}^\top b_i'(v') + \psi \sum_i {\lambda_i^{-}}^\top A_i'(v') \lambda_i^{-} = \ln\left(\frac{2}{\alpha}\right),
\]
then we also achieve $K_t^{\pm}(v) \geq \frac{1}{\alpha}$.
In terms of $v$ and $v'$ these expressions are second degree
equations and thus their real roots in $[0,1]$ (if any) provide 
a safe bracketing of the confidence region $\{v:K_t^{\pm}(v)\leq 1/\alpha\}$. For $K_t^{+}$ let
\begin{align}
C_t&= \sum_{i=1}^{t-1} {\lambda_i^{+}}^\top \cvec{w_i-1}{w_i r_i} \label{eq:upsuffc} ,\\
S_t&= \sum_{i=1}^{t-1} {\lambda_i^{+}}^\top \cvec{0}{1} ,\\
Q_t&= \sum_{i=1}^{t-1} \psi  {\lambda_i^{+}}^\top \symmat{(w_i-1)^2}{(w_i-1)w_i r_i}{w_i^2r_i^2} \lambda_i^{+} ,\\
T_t&= \sum_{i=1}^{t-1} \psi  {\lambda_i^{+}}^\top \symmat{0}{-(w_i-1)}{-2w_ir_i} \lambda_i^{+} ,\\
U_t&=  \sum_{i=1}^{t-1} \psi {\lambda_i^{+}}^\top \symmat{0}{0}{1} \lambda_i^{+}, \label{eq:upsuffu}
\end{align}
and define $C_t',S_t',Q_t',T_t',U_t'$ similarly by using $\lambda_i^{-}$ 
instead of $\lambda_i^{+}$ and $1-r_i$ instead of $r_i$. Then
the largest real root $v^{+}$ of
\[
C_t - S_t v + Q_t + T_t v + U_t v^2 - \ln\left(\frac{2}{\alpha}\right) = 0,
\]
if it exists, satisfies $K_t^{\pm}(v^{+})\geq \frac{1}{\alpha}$. Similarly
we can obtain $v'$ as the largest real root of the quadratic
with $C_t',S_t',Q_t',T_t',U_t'$ in place of $C_t,S_t,Q_t,T_t,U_t$,
if it exists. Then $v^{-}=1-v'$ satisfies $K_t^{\pm}(v^{-})\geq \frac{1}{\alpha}$.

\section{Details of the Scalar Betting Strategy}
\subsection{Elimination of one bet} \label{app:betaopt}
Since in the long term $\lambda_1$ should be 0 
its purpose can only be as a hedge in the short-term. 
We formulate this by considering the
worst case wealth reduction among three outcomes :
$(w,r)=(w_{\max},1)$, 
$(w,r)=(w_{\max},0)$ and $w=0$ with any reward. 
We choose $\lambda_1$ to maximize the wealth 
in the worst of these outcomes. Thus we set 
up a family of Linear Programs (LPs) 
parametrized by $\lambda_2$ and $v$ and with optimization variables $\alpha$ and $\lambda_1$:
\begin{equation*}
\begin{array}{ll@{}ll}
\text{maximize}  & \alpha &\\
\text{subject to}& \alpha \leq 1+\lambda_1(w_{\max}-1)+\lambda_2(w_{\max}-v)  & &(z_1) \\
                 & \alpha \leq 1+\lambda_1(w_{\max}-1)-\lambda_2 v            & &(z_2) \\
                 & \alpha \leq 1-\lambda_1-\lambda_2 v                        & &(z_3), \\
\end{array}
\end{equation*}
where the variable $z_i$ in parentheses next to each constraint is the corresponding dual variable. 
\begin{theorem}
For any $v\in [0,1]$ and any $\lambda_2 \in \R$, the optimal value of $\lambda_1$ in the above LP is $\lambda_1^*=\max(-\lambda_2,0)$.
\end{theorem}
\begin{proof}
The dual program is
\begin{equation*}
\begin{array}{ll@{}ll}
\text{minimize}  & (1+\lambda_2(w_{\max}-v))z_1 +(1-\lambda_2 v) z_2 + (1-\lambda_2 v) z_3 &\\
\text{subject to}& z_i \geq 0  & i=1,2,3 \\
                 & -(w_{\max}-1)(z_1+z_2)+z_3 = 0 \\
                 & z_1+z_2+z_3=1. \\
\end{array}
\end{equation*}
Consider the following two dual feasible settings:
\[
z_1=0,z_2=\frac{1}{w_{\max}}, z_3=\frac{w_{\max}-1}{w_{\max}},
\]
and 
\[
z_1=\frac{1}{w_{\max}}, z_2=0, z_3=\frac{w_{\max}-1}{w_{\max}},
\]
with corresponding dual objectives: $1-\lambda_2 v$ and $1-\lambda_2 v + \lambda_2$. From here we see that if $\lambda_2 > 0$ the former attains a
better dual objective and is thus a better bound
for the primal objective. When $\lambda_2<0$ the latter
is better. 

When $\lambda_2>0$, a primal feasible setting is
$\alpha=1-\lambda_2 v,\lambda_1=0$. Furthermore this setting
achieves the same objective as the first dual feasible setting so 
we conclude that these are the optimal primal and dual solutions when $\lambda_2>0$.

When $\lambda_2<0$, a primal feasible setting is 
$\alpha=1-\lambda_2 v +\lambda_2, \lambda_1=-\lambda_2$. Furthermore this setting
achieves the same objective as the second dual feasible setting
so we conclude that these are the optimal primal and dual solutions when $\lambda_2<0$. 

Finally when $\lambda_2=0$ the two cases give the same value for $\lambda_1$ so we conclude $\lambda_1=\max(-\lambda_2,0)$ for all $\lambda_2 \in \R$ (and $v\geq 0$).
\end{proof}

The theorem suggests that in a hedged strategy the wealth process
eliminating low values of $V(\pi)$ should set 
$\lambda_1^{>} = 0$ because $\E[wr-v] >0$ and thus $\lambda_2^{>} >0$.
The wealth process that eliminates high values of $V(\pi)$
on the other hand should have 
$\lambda_1 = -\lambda_2$ because $\E[wr-v] <0$ and thus $\lambda_2<0$.
Thus the two processes look like
\begin{align*}
K_t^{>}(v)&=\prod_{i=1} \left(1+\lambda_{2,i}^{>} (w_i r_i -v)\right),\\
K_t^{<}(v)&=\prod_{i=1} \left(1-\lambda_{2,i}^{<} (w_i - 1) + \lambda_{2,i}^{<} \left(w_i r_i -v\right)\right)=\prod_{i=1} \left(1-\lambda_{2,i}^{<} \left(w_i (1-r_i) -  (1-v)\right)\right).
\end{align*}
In the main text we have redefined $\lambda_{2,i}^{<}:=-\lambda_{2,i}^{<}$
for symmetry.

\subsection{A Technical Lemma}
\label{app:fan}
The following result can be extracted from the proof of 
Proposition~4.1 in \cite{fan2015exponential}.
\begin{lemma}
For $\xi\geq -1$ and $\lambda \in [0,1)$ we have
\begin{equation}
\ln(1+\lambda \xi) \geq \lambda \xi+\left(\ln\left(1-\lambda\right)+\lambda\right)\cdot \xi^{2}.
\label{eq:fanbound}
\end{equation}
\end{lemma}
\begin{proof}
Note that $\lambda \xi \geq -\lambda > -1$. 
For $x>-1$ the function $f(x) = \frac{\ln(1 + x)-x}{x^2}$
is increasing in $x$, therefore
$f(\lambda \xi) \geq f(-\lambda)$. Rearranging
leads to the statement of the lemma.
\end{proof}
We will be using this lemma with bets 
$\lambda \in [0,1)$ and $\xi_i=w_ir_i-v$ or $\xi_i=w_i(1-r_i)-(1-v_i)$.
In either case $\xi_i\geq -1$.
This lemma provides a stronger lower bound than 
that of Lemma~\ref{lem:quadbound}. The reason
we use the latter for vector bets is that the
natural extension of \eqref{eq:fanbound} to 
the vector case does not lead to a convex 
problem. 

\subsection{Avoiding grid Search}
\label{app:nogrid1d}
Suppose that our bets $\lambda_{2,i}^{+}$ and $\lambda_{2,i}^{-}$ 
do not depend on $v$.
We have the individual lower bounds
\[
\ln(K^{+}(v)) \geq \sum_i \lambda_{2,i}^{+} (w_i r_i -v) + \sum_i (\ln(1-\lambda_{2,i}^{+})+\lambda_{2,i})(w_i r_i -v)^2
\]
and
\[
\ln(K^{-}(v)) \geq \sum_i \lambda_{2,i}^{-} (w_i r_i' -v') + \sum_i (\ln(1-\lambda_{2,i}^{-})+\lambda_{2,i}^{-})(w_i r_i' -v')^2,
\]
where $r'=1-r$, $v'=1-v$.
For the Hedged process, using that for any $a,b$
\[
\ln\left(\exp(a)+\exp(b)\right)\geq \max(a,b)
\]
to first establish
\[
\ln(K^{\pm}(v)) \geq \max(\ln(K^{+}(v))-\ln(2),\ln(K^{-}(v))-\ln(2))
\]
and further bound each term in the maximum by the respective 
quadratic lower bound. We conclude that
if a $v$ achieves 
\[
\sum_i \lambda_{2,i}^{+} (w_i r_i -v) + \sum_i (\ln(1-\lambda_{2,i}^{+})+\lambda_{2,i}^{+})(w_i r_i -v)^2 = \ln\left(\frac{2}{\alpha}\right)
\]
or a $v'=1-v$ achieves 
\[
\sum_i \lambda_{2,i}^{-} (w_i r_i'-v') + \sum_i (\ln(1-\lambda_{2,i}^{-})+\lambda_{2,i}^{-})(w_i r_i' - v')^2
=\ln\left(\frac{2}{\alpha}\right)
\]
then we also achieve $K^{\pm}(v) > \frac{1}{\alpha}$. 
Thus, a valid confidence interval can be obtained by considering
the roots of these quadratics.  Let
\begin{align*}
C&=\sum_i \lambda_{2,i}^{+} w_i r_i & 
C'&=\sum_i \lambda_{2,i}^{-} w_i r_i'\\
S&=\sum_i \lambda_{2,i}^{+} & 
S'&=\sum_i \lambda_{2,i}^{-} \\
Q&=\sum_i \left(\ln(1-\lambda_{2,i}^{+})+\lambda_{2,i}^{+}\right) w_i^2 r_i^2 &
Q'&=\sum_i \left(\ln(1-\lambda_{2,i}^{-})+\lambda_{2,i}^{-}\right) w_i^2 r_i'^2\\
T&=\sum_i \left(\ln(1-\lambda_{2,i}^{+})+\lambda_{2,i}^{+}\right) w_ir_i &
T'&=\sum_i \left(\ln(1-\lambda_{2,i}^{-})+\lambda_{2,i}^{-}\right) w_ir_i'\\
U&=\sum_i \left(\ln(1-\lambda_{2,i}^{+})+\lambda_{2,i}^{+}\right) &
U'&=\sum_i \left(\ln(1-\lambda_{2,i}^{-})+\lambda_{2,i}^{-}\right)\\
\end{align*}
We obtain:
\[
v_{\min}= \frac{2T+S-\sqrt{(2T+S)^2-4U(Q+C-\ln(2/\alpha))}}{2U}
\]
or $v_{\min}=0$ if the discriminant is negative, 
and
\[
v_{\max}=1-v' = 1-\frac{2T'+S'-\sqrt{(2T'+S')^2-4U'(Q'+C'-\ln(2/\alpha))}}{2U'}
\]
or $v_{\max}=1$ if the discriminant is negative.

\section{Reward Predictors}
\label{app:reward-predictors}
\subsection{Betting}
We describe betting for $K_t^{+q}(v)$. Betting for $K_t^{-q}(v)$ 
is analogous. 
We overload the log wealth at step $i$ when betting against $v$ as 
$\ell_i^v(\lambda) = \ln(1+ \lambda_{1,i}(w_i-1) + \lambda_{2,i}(w_i r_i -c_i -v)$. We use lemma~\ref{lem:quadbound} to obtain that for any $\lambda \in \mathcal{E}_v^{1/2}$, we have
\[
\ln(K_t^{+q}(v)) = \sum_{i=1}^{t-1} \ell_i^v(\lambda) \geq \lambda^\top \sum_{i=1}^{t-1} b_i(v) + \psi \lambda^\top \left(\sum_{i=1}^{t-1} A_i(v) \right)\lambda, 
\]
where now $b_i(v)=\cvec{w_i-1}{w_ir_i -c_i -v}$ and 
$A_i(v) = b_i(v)b_i(v)^\top$. As in the case without 
reward predictor we have that the wealth lower bound 
is a polynomial in $v$ with
    \begin{align*}
        \sum_{i=1}^{t-1} A_i(v) &= 
        A_t^{(0)} + v A_t^{(1)} + v^2 A_t^{(2)},\\   
        \sum_{i=1}^{t-1} b_i(v) &= b_t^{(0)} + v b_t^{(1)}, 
    \end{align*}
and the coefficients can be maintained as
\allowdisplaybreaks
    \begin{align*}
        A_t^{(0)} &=\sum_{i=1}^{t-1}\symmat{(w_i-1)^2}{(w_i-1)(w_i r_i-c_i)}{(w_ir_i-c_i)^2}, \\
        A_t^{(1)} &= \sum_{i=1}^{t-1} \symmat{0}{-(w_i-1)}{-2(w_ir_i-c_i)},\\
        A_t^{(2)} &=\sum_{i=1}^{t-1}  \symmat{0}{0}{1},\\
        b_t^{(0)} &=\sum_{i=1}^{t-1}  \cvec{w_i-1}{w_ir_i-c_i},\\
        b_t^{(1)} &=\sum_{i=1}^{t-1}  \cvec{0}{-1}.
    \end{align*}
Given a $v$ we compute concrete values for these coefficients 
and then solve 
\[
\lambda_t = \argmax_{\lambda \in \mathcal{E}^{1/2}} 
\lambda^\top \sum_{i=1}^{t-1} b_i(v) + \psi \lambda^\top \left(\sum_{i=1}^{t-1} A_i(v) \right)\lambda.
\]
A similar procedure like the one in Algorithm~\ref{alg:argmax} 
can then be used for solving this problem.

\subsection{Avoiding Grid Search}
To find the value of $v$ that we can plug in to the above 
optimization problem we proceed as in section~\ref{sec:avoid-grid},
and further explained in Appendix~\ref{app:nogrid2d}.
To find a $v$ such that 
$K_t^{\pm q}(v)\geq \frac{1}{\alpha}$
it suffices to solve
\[
\sum_{i=1}^{t-1} \lambda_i^\top b_i(v) + \psi \sum_{i=1}^{t-1} \lambda_i^\top  A_i(v) \lambda_i = \ln\left(\frac{2}{\alpha}\right) ,
\]
given the previous bets $\lambda_1,\ldots, \lambda_{t-1}$. This 
is a second degree equation which can be solved by maintaining 
the quantities
\begin{align*}
C_t&= \sum_{i=1}^{t-1} {\lambda_i}^\top \cvec{w_i-1}{w_i r_i-c_i} , \\
S_t&= \sum_{i=1}^{t-1} {\lambda_i}^\top \cvec{0}{1} , \\
Q_t&= \sum_{i=1}^{t-1} \psi  {\lambda_i}^\top \symmat{(w_i-1)^2}{(w_i-1)(w_i r_i-c_i)}{(w_ir_i-c_i)^2} \lambda_i , \\
T_t&= \sum_{i=1}^{t-1} \psi  {\lambda_i}^\top \symmat{0}{-(w_i-1)}{-2(w_ir_i-c_i)} \lambda_i , \\
U_t&=  \sum_{i=1}^{t-1} \psi {\lambda_i}^\top \symmat{0}{0}{1} \lambda_i ,
\end{align*}
and finding the largest real root $v$ of
\[
C_t - S_t v + Q_t + T_t v + U_t v^2 - \ln\left(\frac{2}{\alpha}\right) = 0,
\]
if it exists, otherwise setting $v=0$.

\subsection{Double Hedging}
Double Hedging boils down to running four processes: 
$K_t^{+q}, K_t^{-q}, K_t^{+}$, and $K_t^{-}$. Note 
that the wealth is split in 4 so anywhere we used 
$\ln\left(\frac{2}{\alpha}\right)$ in a hedged process
now we need to use $\ln\left(\frac{4}{\alpha}\right)$.
Note that both $K_t^{+q}(v)$ and $K_t^{+}(v)$ are trying to 
establish bounds for the same random variable and in principle
they could communicate about values that
have been eliminated. However we keep things simple
and just run the four processes without sharing 
any information. The wealth of the doubly 
hedged process can then be lower bounded by 
the wealth of the most successful betting 
strategy starting from a wealth of $\frac{1}{4}$.

\section{Gated Deployment}
\label{app:gd}
\subsection{Hedging}
Since we don't typically know whether $\pi$ is better or 
worse that $h$ we can hedge our bets via the process
\[
K_{t}^{\pm gd}(v) =\frac{1}{2} (K_{t}^{+gd}(v)+K_{t}^{-gd}(v)),
\]
where 
\begin{align*}
    K_{t}^{+ gd}(v) &=\prod_{i=1}^t \left(1+\lambda_{1,i}^{+} (w_i-1) + \lambda_{2,i}^{+} (w_ir_i -r_i -v)\right),  \\
    K_{t}^{- gd}(v) &=\prod_{i=1}^t \left(1+\lambda_{1,i}^{-} (w_i-1) + \lambda_{2,i}^{-} (w_ir_i' -r_i' -v')\right) ,
\end{align*}
for predictable 
$\lambda_{1,i}^{+},\lambda_{2,i}^{+}, \lambda_{1,i}^{-},\lambda_{2,i}^{-}$
subject to $\lambda_i^{+},\lambda_{i}^{-} \in \mathcal{G}_{v}^0$.
As before, $r_i'=1-r_i$ and $v'=1-v$. 

\subsection{Betting and Avoiding Grid Search}
Betting and avoiding grid search can be obtained using 
the same equations as for reward predictors but replacing 
all occurrences of $c_i$ with $r_i$.

A key difference we spell out is the feasible region. In
order to use common bets and to be able to use the 
quadratic lower bound of the log wealth we need to 
specify the set $\bigcap_{v\in [0,1]} \mathcal{G}_v^{m}$.
This set is equivalent to
\begin{align*}
\mathcal{G}^m = \left\{\lambda:  
\left[\begin{array}{cc}
-1 & -2 \\ 
-1 & 0 \\
W &  -1\\
W & W
\end{array}\right]
\lambda \geq m - 1 \label{eq:gd-explicit-domain}
\right\},
\end{align*}
where $W = w_{\max}-1$. If we further restrict $\lambda_2 \geq 0$
for each of the subprocesses because we expect each to eliminate 
$v$ such that $\E[wr-v]>0$ and $v'$ such that $\E[wr'-v']>0$ 
then the feasible region further simplifies to 
\[
\mathcal{G} = \{\lambda: \lambda_2\geq 0,
W\lambda_1 -\lambda_2 \geq m-1, 
-\lambda_1 -2 \lambda_2 \geq m-1 \}.
\]
Placing bets in this region can be done using the same 
ideas as Algorithm~\ref{alg:argmax}.

\section{Reproducibility Checklist}
\begin{description}
\item Assumptions: The contextual bandit data is iid. 
The policy $\pi$ is absolutely 
continuous with respect to behavior policy $h$.
\item Complexity: MOPE and the scalar Betting Strategy are streaming algorithms. They require constant time per sample and constant memory
independent of number of samples. The exact wealth ablation requires
memory that scales linearly with the number of samples and time 
per step that scales at least linearly with the number of samples.
The ablation that solves a QP per value $v$ requires at least $\frac{1}{\epsilon}$ 
times more memory and computation that MOPE and provides 
results that are accurate up to $\epsilon$. We used 
$\epsilon=0.005$ in the experiments.
\item Code: included with the supplementary material and 
will be released publicly upon acceptance.
\item Data: synthetic environments are part of the code.
Instructions for getting the mnist8m data are in the 
``Mnist-Policies'' notebook.
\item Hyperparameters: There are no hyperparameters.
The confidence level is an input and is stated in each experiment
description or the corresponding figure.
\item Computing infrastructure: Off-the-shelf workstation running Linux
(Code works on a Windows laptop as well).
\end{description}

\end{document}